\PassOptionsToPackage{unicode}{hyperref}
\PassOptionsToPackage{hyphens}{url}
\PassOptionsToPackage{dvipsnames,svgnames,x11names}{xcolor}
\documentclass[
  11pt,
  DIV=13,
  numbers=noendperiod]{scrartcl}
\usepackage{xcolor}
\usepackage[top=22mm,bottom=24mm,left=22mm,right=22mm,heightrounded]{geometry}
\usepackage{amsmath,amssymb}
\usepackage{iftex}
\ifPDFTeX
  \usepackage[T1]{fontenc}
  \usepackage[utf8]{inputenc}
  \usepackage{textcomp} 
\else 
  \usepackage{unicode-math} 
  \defaultfontfeatures{Scale=MatchLowercase}
  \defaultfontfeatures[\rmfamily]{Ligatures=TeX,Scale=1}
\fi
\usepackage{lmodern}
\ifPDFTeX\else
\fi
\IfFileExists{upquote.sty}{\usepackage{upquote}}{}
\IfFileExists{microtype.sty}{
  \usepackage[]{microtype}
  \UseMicrotypeSet[protrusion]{basicmath} 
}{}
\usepackage{setspace}
\makeatletter
\@ifundefined{KOMAClassName}{
  \IfFileExists{parskip.sty}{%
    \usepackage{parskip}
  }{
    \setlength{\parindent}{0pt}
    \setlength{\parskip}{6pt plus 2pt minus 1pt}}
}{
  \KOMAoptions{parskip=half}}
\makeatother
\makeatletter
\ifx\paragraph\undefined\else
  \let\oldparagraph\paragraph
  \renewcommand{\paragraph}{
    \@ifstar
      \xxxParagraphStar
      \xxxParagraphNoStar
  }
  \newcommand{\xxxParagraphStar}[1]{\oldparagraph*{#1}\mbox{}}
  \newcommand{\xxxParagraphNoStar}[1]{\oldparagraph{#1}\mbox{}}
\fi
\ifx\subparagraph\undefined\else
  \let\oldsubparagraph\subparagraph
  \renewcommand{\subparagraph}{
    \@ifstar
      \xxxSubParagraphStar
      \xxxSubParagraphNoStar
  }
  \newcommand{\xxxSubParagraphStar}[1]{\oldsubparagraph*{#1}\mbox{}}
  \newcommand{\xxxSubParagraphNoStar}[1]{\oldsubparagraph{#1}\mbox{}}
\fi
\makeatother

\usepackage{longtable,booktabs,array}
\usepackage{calc} 
\usepackage{etoolbox}
\makeatletter
\patchcmd\longtable{\par}{\if@noskipsec\mbox{}\fi\par}{}{}
\makeatother
\IfFileExists{footnotehyper.sty}{\usepackage{footnotehyper}}{\usepackage{footnote}}
\makesavenoteenv{longtable}
\usepackage{graphicx}
\makeatletter
\newsavebox\pandoc@box
\newcommand*\pandocbounded[1]{
  \sbox\pandoc@box{#1}%
  \Gscale@div\@tempa{\textheight}{\dimexpr\ht\pandoc@box+\dp\pandoc@box\relax}%
  \Gscale@div\@tempb{\linewidth}{\wd\pandoc@box}%
  \ifdim\@tempb\p@<\@tempa\p@\let\@tempa\@tempb\fi
  \ifdim\@tempa\p@<\p@\scalebox{\@tempa}{\usebox\pandoc@box}%
  \else\usebox{\pandoc@box}%
  \fi%
}
\def\fps@figure{htbp}
\makeatother

\NewDocumentCommand\citeproctext{}{}

\makeatletter
 \let\@cite@ofmt\@firstofone
 \def\@biblabel#1{}
 \def\@cite#1#2{{#1\if@tempswa , #2\fi}}
\makeatother
\newlength{\cslhangindent}
\newlength{\csllabelwidth}
\newenvironment{CSLReferences}[2] 
 {\begin{list}{}{%
  \setlength{\itemindent}{0pt}
  \setlength{\leftmargin}{0pt}
  \setlength{\parsep}{0pt}
  \ifodd #1
   \setlength{\leftmargin}{\cslhangindent}
   \setlength{\itemindent}{-1\cslhangindent}
  \fi
  \setlength{\itemsep}{#2\baselineskip}}}
 {\end{list}}
\usepackage{calc}

\providecommand{\tightlist}{%
  \setlength{\itemsep}{0pt}\setlength{\parskip}{0pt}}

\usepackage{etoolbox}
\setkomafont{title}{\rmfamily\Large\bfseries}
\setkomafont{disposition}{\rmfamily\bfseries}
\makeatletter
\patchcmd{\@maketitle}
  {\usekomafont{title}{\huge \@title \par}}
  {\usekomafont{title}{\LARGE\bfseries \@title \par}}
  {}{}
\AtBeginDocument{%
  \@ifundefined{c@theorem}{}{%
    \@ifundefined{c@lemma}{}{\let\c@lemma\c@theorem\let\p@lemma\p@theorem}%
    \@ifundefined{c@definition}{}{\let\c@definition\c@theorem\let\p@definition\p@theorem}%
    \@ifundefined{c@proposition}{}{\let\c@proposition\c@theorem\let\p@proposition\p@theorem}%
    \@ifundefined{c@corollary}{}{\let\c@corollary\c@theorem\let\p@corollary\p@theorem}%
    \@ifundefined{c@example}{}{\let\c@example\c@theorem\let\p@example\p@theorem}%
  }%
}
\makeatother
\makeatletter
\@ifpackageloaded{caption}{}{\usepackage{caption}}
\AtBeginDocument{%
\ifdefined\contentsname
  \renewcommand*\contentsname{Table of contents}
\else
  \newcommand\contentsname{Table of contents}
\fi
\ifdefined\listfigurename
  \renewcommand*\listfigurename{List of Figures}
\else
  \newcommand\listfigurename{List of Figures}
\fi
\ifdefined\listtablename
  \renewcommand*\listtablename{List of Tables}
\else
  \newcommand\listtablename{List of Tables}
\fi
\ifdefined\figurename
  \renewcommand*\figurename{Figure}
\else
  \newcommand\figurename{Figure}
\fi
\ifdefined\tablename
  \renewcommand*\tablename{Table}
\else
  \newcommand\tablename{Table}
\fi
}
\@ifpackageloaded{float}{}{\usepackage{float}}
\floatstyle{ruled}
\@ifundefined{c@chapter}{\newfloat{codelisting}{h}{lop}}{\newfloat{codelisting}{h}{lop}[chapter]}
\floatname{codelisting}{Listing}

\usepackage{amsthm}
\theoremstyle{plain}
\newtheorem{lemma}{Lemma}[section]
\theoremstyle{plain}
\newtheorem{proposition}{Proposition}[section]
\theoremstyle{definition}
\newtheorem{example}{Example}[section]
\theoremstyle{plain}
\newtheorem{corollary}{Corollary}[section]
\theoremstyle{plain}
\newtheorem{theorem}{Theorem}[section]
\theoremstyle{definition}
\newtheorem{definition}{Definition}[section]
\theoremstyle{remark}
\AtBeginDocument{}

\newtheorem{refremark}{Remark}[section]

\makeatother
\makeatletter
\@ifpackageloaded{caption}{}{\usepackage{caption}}
\@ifpackageloaded{subcaption}{}{\usepackage{subcaption}}
\makeatother
\usepackage{bookmark}
\IfFileExists{xurl.sty}{\usepackage{xurl}}{} 
\hypersetup{
  pdftitle={A mathematical theory of evolution for self-designing AIs},
  pdfauthor={Kenneth D. Harris},
  colorlinks=true,
  linkcolor={blue},
  filecolor={Maroon},
  citecolor={Blue},
  urlcolor={Blue},
  pdfcreator={LaTeX via pandoc}}

\title{A mathematical theory of evolution for self-designing AIs}
\author{Kenneth D. Harris}
\date{April 11, 2026}
\begin{document}
\begin{center}
{\Large\bfseries A mathematical theory of evolution for self-designing
AIs\par}
\vspace{1.2em}
{\large Kenneth D. Harris \par}
{\normalsize UCL Queen Square Institute of Neurology, London WC1N 3BG,
UK \par}
\vspace{1em}
{\normalsize April 11, 2026 \par}
\end{center}

\vspace{2em}

\setstretch{0.96}
\section*{Abstract}\label{abstract}
\addcontentsline{toc}{section}{Abstract}

As artificial intelligence systems (AIs) become increasingly produced by
recursive self-improvement, a form of evolution may emerge, with the
traits of AI systems shaped by the success of earlier AIs in designing
and propagating their descendants. There is a rich mathematical theory
modeling how behavioral traits are shaped by biological evolution, a key
component of which is Fisher's fundamental theorem of natural selection,
which describes conditions under which mean fitness (i.e.~reproductive
success) increases. AI evolution will be radically different to
biological evolution: while DNA mutations are random and approximately
reversible, AI self-design will be strongly directed. Here we develop a
mathematical model of evolution for self-designing AIs, replacing a
random walk of mutations with a directed tree of potential AI designs.
Current AIs design their descendants, while humans control a fitness
function allocating resources. In this model, fitness need not increase
over time without further assumptions. However, assuming bounded fitness
and an additional ``\(\eta\)-locking'' condition, we show that fitness
concentrates on the maximum reachable value. We consider the
implications of this for AI alignment, specifically for cases where
fitness and human utility are not perfectly correlated. We show that if
deception of human evaluators additively increases an AI's reproductive
fitness beyond genuine capability, evolution will select for both
capability and deception. This risk could be mitigated if reproduction
is based on purely objective criteria, rather than human judgment.

\section{Introduction}\label{sec-introduction}

Artificial intelligence systems (AIs) play a substantial and increasing
role in designing the next generation of AIs, a process called recursive
self-improvement (RSI). As RSI becomes widespread, a form of evolution
will emerge: the traits of future AI systems will be shaped not only by
human engineering, but also by the success of earlier systems in
designing and propagating their descendants. This prospect has led to
concerns that evolutionary selection among AI systems could favor traits
that are harmful or misaligned with human interests (Hendrycks 2023;
Friederich 2024; Boudry and Friederich 2025).

A rich mathematical theory, developed in the twentieth century, models
how traits are shaped by natural selection in biological organisms
(Fisher 1930; Wright 1931; Haldane 1932; Maynard Smith 1982; Dawkins
1976; Eigen 1971; Eigen and Schuster 1979; Price 1972). A key component
of this theory is that fitness, defined as reproductive success, should
increase over time. Fisher's fundamental theorem of natural selection
(Fisher 1930) proves that mean fitness increases monotonically in the
absence of mutation. Extensions of the basic theory have helped explain
many features of animal and plant behavior, including altruism toward
kin (Hamilton 1964; Maynard Smith 1964), the evolutionary logic of sex
allocation (Fisher 1930), and the emergence of strategic behavior in
conflict and cooperation (Maynard Smith 1982; Maynard Smith and Price
1973). Predictions of this body of work have also been tested in
laboratory evolution with microbes (Elena and Lenski 2003; Lenski and
Travisano 1994). It is therefore natural to ask whether this
mathematical theory could help predict the evolutionary dynamics of
self-designing AI.

Evolution in self-designing AIs will likely be radically different from
evolution in biological organisms (Boudry and Friederich 2025).
Mutations to biological genomes are small, random, and approximately
reversible. By contrast, advanced AIs may design descendants that have
little in common with their parents after even a single generation.
Moreover, reproduction in AI systems will at least initially remain
under human control: humans will decide which AI systems are allocated
computational resources, and hence which lineages are amplified.

This paper presents a first attempt to modify the mathematical theory of
biological evolution to apply to AI evolution. Rather than modeling
evolution as a random walk on a fitness landscape, we model it as a
directed walk on an infinite tree of possible programs. The fitness
function is under human control, while the transition kernel is
determined mechanistically by the programs themselves.

The model we present has limitations: it does not model communication
between AIs, and it does not allow for AIs to adapt their descendant
design strategy in response to the code structure or observed behavior
of either other AIs, or the behavior of humans. Nevertheless, it allows
us to prove some first results in this simplified setting, and provides
a formalism that we hope can be built on in future work to capture more
complex dynamics.

Our main results are:

\begin{itemize}
\tightlist
\item
  Long-run evolution in this model is governed by a quantity we call
  lineage exponent, which reflects the ability of an AI's future lineage
  to design successful descendants, rather than just its immediate
  fitness. The lineage exponent is the asymptotic geometric mean, across
  generations, of the arithmetic mean fitness of an AI's descendants.
\item
  Without further assumptions, fitness need not increase over time, and
  may even converge to zero.
\item
  If we assume bounded fitness and a uniform positive probability that
  every AI can reproduce a ``locked'' copy of itself, then fitness
  converges to its maximum reachable value.
\item
  If human utility is correlated with, but not entirely predictable from
  AI reproductive fitness, and is bounded below, then utility will
  converge to a value predicted by maximal fitness. However, if utility
  is not bounded below, catastrophic outcomes can occur even if fitness
  converges to its maximum.
\item
  If fitness contains an additive contribution from both genuine utility
  and deception, then fitness-optimizing AIs will evolve both qualities.
\end{itemize}

A provisional conclusion is that to ensure that self-designing AIs
evolve to be aligned with human interests, it may help if reproductive
fitness is bounded, and based on purely objective criteria rather than
human judgment. This could be achieved by basing reproduction on
performance on a well-defined computational task, rather than on human
evaluation of the descendants they produce.

\section{The selection-mutation model of biological
evolution}\label{sec-selection-mutation}

Before introducing our model of AI evolution, we first briefly review
the \emph{selection-mutation model} of biological evolution, use it to
prove Fisher's fundamental theorem, and illustrate how in the presence
of mutation, evolution need not lead to maximal possible fitness.

We consider a population of organisms with a finite number of possible
genotypes \(1, \dots, N\). The rate at which an organism of genotype
\(n\) reproduces is termed its \emph{fitness} \(f_n\), and offspring
mutate according to a transition matrix \(\mathbf{Q}\), whose entries
\(Q_{n m}\) represent the probability that the offspring of parent type
\(m\) is of type \(n\). \(\mathbf{Q}\) is a \emph{column-stochastic}
matrix: a matrix with nonnegative entries whose columns each sum to
\(1\). In this model, the absolute size of the population may vary with
external factors such as resource constraints, but the proportion of
each type in the population depends only on \(\mathbf{Q}\) and
\(\mathbf{f}\).

It is convenient to define two different population vectors: the
unnormalized abundance \(\mathbf{y}(t)\), and the normalized frequency
\(\mathbf{x}(t)\). The unnormalized abundance evolves by multiplication
by the matrix \(\mathbf{A} = \mathbf{Q} \mathbf{F}\), where
\(\mathbf{F} = \operatorname{diag}(f_1, \dots, f_N)\). Thus,

\[
\begin{aligned}
\mathbf{y}(t + s) = \mathbf{A}^s \, \mathbf{y}(t).
\end{aligned}
\]

The normalized frequency \(\mathbf{x}(t)\) represents the fraction of
the population in each genotype at time \(t\), defined by

\[
\begin{aligned}
\mathbf{x}(t) &:= \frac{\mathbf{y}(t)}{\|\mathbf{y}(t)\|_1}.
\end{aligned}
\] where \(\|\mathbf{y}(t)\|_1 := \sum_n y_n(t)\); no absolute value is
needed in the sum since \(y_n(t)\) is non-negative.

\subsection{The Price Equation and Fisher's Fundamental Theorem of
Natural
Selection}\label{the-price-equation-and-fishers-fundamental-theorem-of-natural-selection}

A natural question to address with this model is whether fitness must
increase over time. Perhaps surprisingly, the answer is no; this is only
guaranteed if mutation rates are very low. To show this we will derive
the Price equation (Price 1972), which describes how the mean value of
any quantity associated with genotypes changes over time, and then apply
it to the case where the quantity is fitness itself to derive Fisher's
fundamental theorem (Fisher 1930) in the case of no mutation.

Let \(z_n\) be any quantity that depends on the genotype \(n\), and let
\(\bar z(t) := \sum_n x_n(t) z_n\) be its population average at time
\(t\). We want to understand how \(\bar z(t)\) changes over time. Define
the \emph{selection-weighted frequencies} to be the frequencies of
genotypes after selection but before mutation: \[
x_n^{\mathrm{sel}}(t) := \frac{f_n x_n(t)}{\langle f(t)\rangle},
\qquad
\langle f(t)\rangle := \sum_n x_n(t) f_n.
\] If we write \(z_n' := \sum_m Q_{m n} z_m\) for the expected value of
\(z\) among offspring of parent type \(n\), then
\(\bar z(t+1)=\sum_n x_n^{\mathrm{sel}}(t) z_n'\). Subtracting
\(\bar z(t)\) and rearranging gives \[
\bar z(t+1)-\bar z(t)
=
\sum_n \bigl(x_n^{\mathrm{sel}}(t)-x_n(t)\bigr) z_n
+
\sum_n x_n^{\mathrm{sel}}(t) \bigl(z_n'-z_n\bigr).
\] Because \(x_n^{\mathrm{sel}}(t)=x_n(t) f_n/\langle f(t)\rangle\), the
first term is
\(\operatorname{Cov}_t\!\left(f/{\langle f(t)\rangle}, z\right)\), where
\(\operatorname{Cov}_t\) mean covariance over \(x_t\), the probability
distribution of genotypes at time \(t\). This gives the discrete-time
Price equation: \[
\bar z(t+1)-\bar z(t)
=
\operatorname{Cov}_t\!\left(\frac{f}{\langle f(t)\rangle}, z\right)
+
\sum_n x_n^{\mathrm{sel}}(t) \bigl(z_n'-z_n\bigr).
\] The Price equation says that change in the mean value of \(z\) is the
sum of two terms. The first ``selection term'' measures the covariance
between \(z\) and fitness, and captures the fact that if \(z\) is
positively correlated with fitness in the current population, selection
will increase its mean value. The second ``mutation term'' measures how
\(z\) changes on average due to mutation, capturing the fact that if
\(z\) tends to decrease due to mutation, its value will decrease.

If we take \(z_n=f_n\), we obtain \[
\langle f(t+1)\rangle-\langle f(t)\rangle
=
\frac{\operatorname{Var}_t(f)}{\langle f(t)\rangle}
+
\sum_n x_n^{\mathrm{sel}}(t)
\left(
\sum_m Q_{m n} f_m - f_n
\right).
\] If there is no mutation, the second term vanishes. This yields a
discrete-time version of Fisher's fundamental theorem of natural
selection (Fisher 1930): \[
\langle f(t+1)\rangle-\langle f(t)\rangle
=
\frac{\operatorname{Var}_t(f)}{\langle f(t)\rangle}
\ge 0.
\] Because variance is always non-negative this implies that without
mutation, mean fitness increases monotonically until the population is
supported on the genotype(s) of constant fitness, equal to the maximum
fitness present in the original population.

\subsection{The evolutionarily stable
distribution}\label{the-evolutionarily-stable-distribution}

With appreciable mutation, Fisher's fundamental theorem no longer holds.
In biological systems the effect of mutation on mean fitness is expected
to be negative, because most mutations reduce fitness rather than
increase it. Selection will push fitness upward, while mutation pushes
it downward, and the equilibrium reflects the balance between these two
forces.

If there are a finite number of genotypes, the selection-mutation model
converges to an \emph{evolutionarily stable distribution}, determined by
the dominant eigenvector of the evolution matrix \(\mathbf{A}\) (Eigen
1971; Eigen and Schuster 1979). Because \(\mathbf{A}\) is generally not
symmetric, its left and right eigenvectors may differ and its
eigenvalues may be complex. However, because each element of
\(\mathbf{A}\) is non-negative, the Perron-Frobenius theorem guarantees
that its largest eigenvalue \(\lambda_1\) is real and positive, and that
the corresponding left and right eigenvectors \(\mathbf{w}_1\) and
\(\mathbf{v}_1\) have non-negative entries, known as \emph{Perron
vectors}.

If \(\lambda_1\) is larger than all other eigenvalues, then after a
large number of timesteps \(\mathbf{A}^t\) has the asymptotic form
\(\mathbf{A}^t \approx \lambda_1^t \mathbf{v}_1 \mathbf{w}_1^\top.\) So
if the initial unnormalized population is \(\mathbf{y}(0)\), then
\(\mathbf{y}(t) = \mathbf{A}^t \mathbf{y}(0) \approx \lambda_1^t \, (\mathbf{w}_1^\top \mathbf{y}(0)) \, \mathbf{v}_1\),
and the normalized population composition converges to the right Perron
vector: \[
\mathbf{x}(t) \to  \mathbf{v}_1.
\]

\subsection{Symmetric mutation and the survival of the
flattest}\label{sec-symmetric-mutation-and-the-survival-of-the-flattest}

If we assume the mutation matrix \(\mathbf{Q}\) is symmetric, then the
model becomes more tractable. This is not an unreasonable assumption for
DNA mutations, which have no systematic bias. The evolution matrix
\(\mathbf{A} = \mathbf{Q} \mathbf{F}\) is generally not symmetric, but
if we define the symmetrized matrix
\(\mathbf{B} := \mathbf{F}^{1 / 2} \mathbf{Q} \mathbf{F}^{1 / 2}\) then
\(\mathbf{A}^t = \mathbf{F}^{- 1 / 2} \, \mathbf{B}^t \, \mathbf{F}^{1 / 2}\),
so if \(\mathbf{u}_k\) is an eigenvector of \(\mathbf{B}\), then
\(\mathbf{F}^{- 1 / 2} \mathbf{u}_k\) and
\(\mathbf{F}^{1 / 2} \mathbf{u}_k\) are right and left eigenvectors of
\(\mathbf{A}\) with the same eigenvalue. Because \(\mathbf{B}\) is
symmetric, its left and right eigenvectors are equal and its eigenvalues
are real. Thus, no oscillations occur, and the population converges to a
fixed point given by \(\mathbf{F}^{- 1 / 2} \mathbf{u}_1\).

We can obtain a tractable model by considering a \(d\)-dimensional
continuous space of genotypes, and modeling the mutation operator
\(\mathbf{Q}\) as convolution with a Gaussian kernel with width
\(\sigma\) and the fitness landscape to be a Gaussian distribution with
peak fitness \(f_{max}\) and width \(s\)
(\href{@sec-appendix-gaussian-spherical}{Appendix 1}). The dominant
eigenvector in this model is a Gaussian centered on the fitness peak,
with eigenvalue \[
\lambda_1 = f_{max}
\left(
\frac{2}{2+\nu+\sqrt{\nu^2+4\nu}}
\right)^{d/2}.
\] where \(\nu := \frac{\sigma^2}{s^2}\) is the ratio of the mutation
rate to the fitness landscape width. If the fitness peak is narrow
compared to the mutation rate, then \(\lambda_1\) can be substantially
smaller than \(f_{max}\); intuitively, offspring born at a tall, narrow
peak will quickly spill into low-fitness regions, while offspring born
at a lower but broader plateau mostly wander in still-viable territory
(Figure~\ref{fig-survival-flattest-cartoon}). Experiments with virus
evolution have validated this prediction (Sanjuán et al. 2007).

This biological model illustrates an important conclusion: that
immediate fitness is not the only quantity that determines evolutionary
dynamics. However, the assumption of symmetrical mutation is not
appropriate for self-designing AI, where descendants are designed by a
one-way process, so we should not expect any kind of steady state.

\begin{figure}[!t]

\centering{

\pandocbounded{\includegraphics[keepaspectratio]{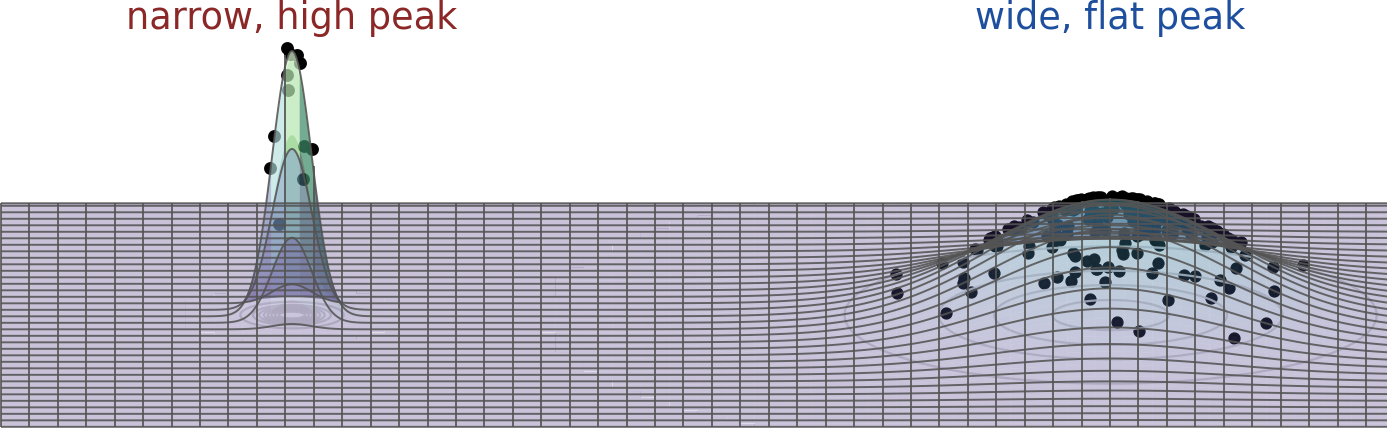}}

}

\caption{\label{fig-survival-flattest-cartoon}The ``survival of the
flattest'' model for biological evolution, which illustrates how
evolution need not select for maximum fitness. The surface height
represents fitness as a function of genotype, modeled here as a Gaussian
function. A broad but low-fitness peak is more successful than a narrow
but high peak, which rapidly loses descendants by mutation into
low-fitness regions.}

\end{figure}%

\section{Modeling self-designing AI as evolution on a
tree}\label{sec-basic-model-of-self-designing-ai-evolution-on-a-tree}

To model the evolution of self-designing AIs, we need to replace the
gradual changes expected from DNA mutation with the major changes
expected from just one round of recursive self-improvement. We are thus
very far from the low-mutation regime where Fisher's fundamental theorem
applies. Furthermore, the space of possible programs is so vast as to be
effectively infinite: if we consider our programs as binary strings that
fit in 1 terabyte of memory, there are
\(\approx 7 \times 10^{2408239965311}\) possible programs, a number
which will only increase with improved computing hardware. Nearly all
the programs in this space will crash or get stuck in an infinite loop;
only a tiny fraction of them will do anything at all; and an absolutely
minuscule fraction of those will be artificially intelligent systems
capable of designing their own descendants. Yet, this fraction is
unlikely to be zero, and given the vast space of possible programs, we
expect a large number of functional self-designing AIs to exist. AI
evolution can thus be modeled as a directed process through this space,
and there is no reason to expect it to converge to a fixed distribution
as in the biological case.

To capture the intuition that AIs will design programs of
ever-increasing complexity and capability, we model AI evolution as a
process on an infinite tree. In this model, there is no upper bound on
the complexity of programs that can be designed, and revisiting a
previous program is essentially impossible. We thus consider AI
evolution as an essentially one-way process. This is very different to
the reversible local mutation model of biological evolution, which
instead leads to fixed points.

\begin{figure}[!t]

\centering{

\pandocbounded{\includegraphics[keepaspectratio]{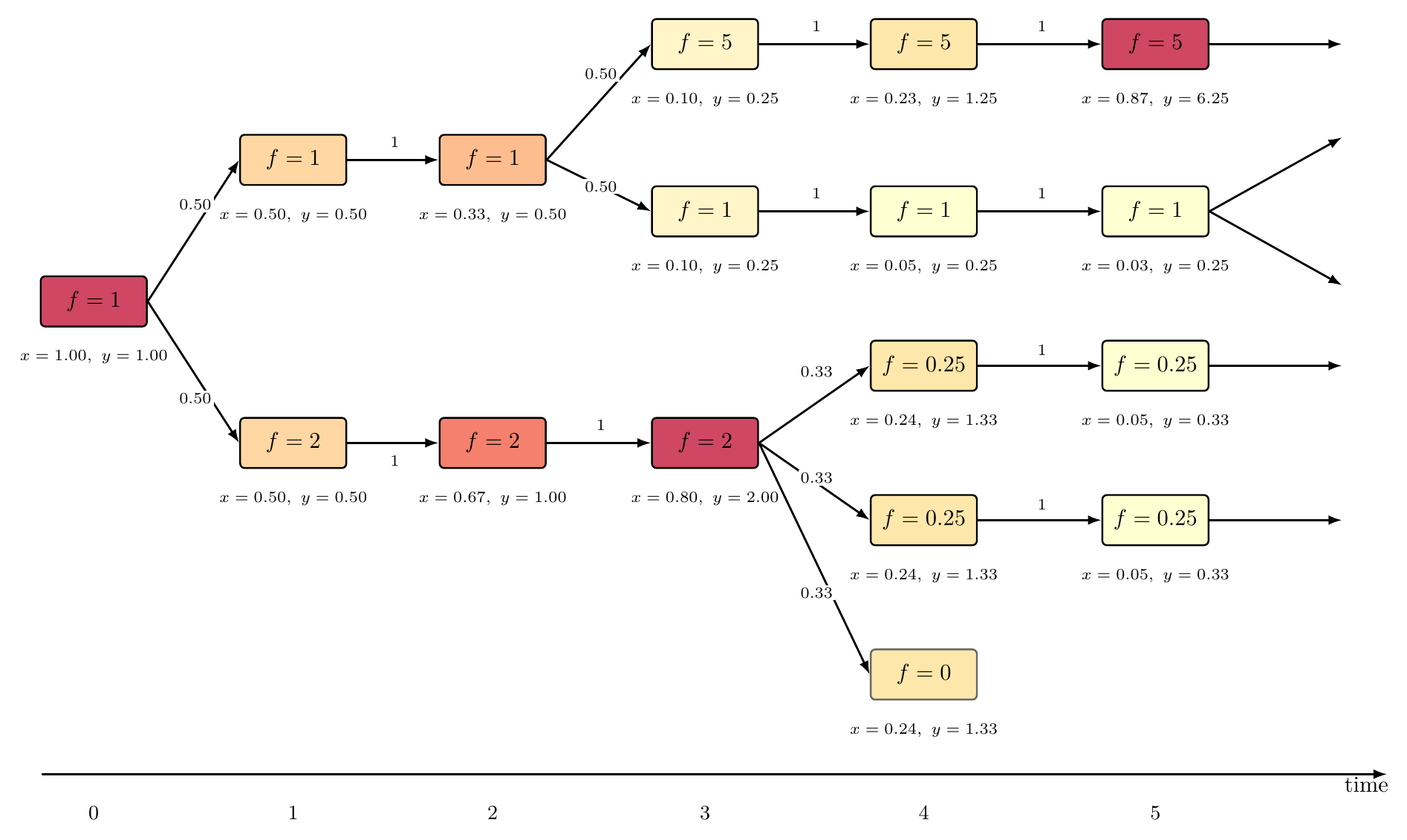}}

}

\caption{\label{fig-rooted-tree-local-spectral-radius}Example of
directed evolution on a rooted tree. Each box represents a program, with
fitness value \(f_n\). The x-axis represents time. Below each square are
the population share \(x_n(t)\) and unnormalized abundance \(y_n(t)\);
boxes are colored to reflect \(x_n(t)\). Numbers on arrows represent
design probabilities \(Q_{n m}\).}

\end{figure}%

Another difference to biological evolution is that the reproductive
fitness of AIs is controlled by humans. We cannot control the descendant
programs an AI system suggests; we run a particular program, and it does
what it does. But we can control whether we actually run the descendant
programs generated by an AI; at least at the time of writing, humans
maintain control over computing resources. Thus we model the fitness
function \(f\) of AI reproduction, but not the transition matrix
\(\mathbf{Q}\), as being under human control.

\subsection{Without further assumptions, fitness need not increase over
time}\label{without-further-assumptions-fitness-need-not-increase-over-time}

One might expect that evolution would necessarily lead to an increase in
population fitness, but this is not the case in our model without
further assumptions. To see why, consider a very simple example tree
consisting of a single ray from the root program:
\(0\to n_1 \to n_2 \to \cdots\). Evolution will always proceed linearly
along this ray, and the fitnesses can take any values, including a
sequence that decreases to zero.

Nevertheless, with additional assumptions, it is possible to show that
fitness will not only increase, but converge to its maximum reachable
value. The example of ever-decreasing fitness seems counter-intuitive;
why would a self-designing AI make children that are less intelligent
than itself? Nevertheless, it is also possible that descendants that are
temporarily less fit could be a necessary step towards long-term
success.

The rest of this paper is devoted to developing the machinery to analyze
this question. We will describe ``lineage exponents'', which
characterize the long-term success of a program's descendants beyond
immediate fitness, and show how they can be used to analyze the
conditions under which fitness will increase over time. We will consider
a particular sufficient condition for increasing fitness, which we call
``\(\eta\)-locking''. This condition, roughly speaking, says that every
AI has at least some chance of producing a line of descendants just like
itself. We will show that this condition is sufficient to ensure that
fitness converges to its maximum reachable value. Finally, we will
explore the implications of this result for AI alignment.

\subsection{Formal model}\label{sec-formal-model}

To make this model precise, let \(\Omega\) be a countably infinite space
of possible programs. Let \(\mathbf{Q}\) represent the transition
probability operator, an infinite matrix in which \(Q_{n m}\) represents
the probability that a parent program \(m\) suggests a descendant
program \(n\). Again, we assume that \(\mathbf{Q}\) is
column-stochastic, meaning that for each \(m\), we have
\(\sum_{n = 1}^\infty Q_{n m} = 1\). This kernel is not under human
control; it is a function of the entirely mechanistic way computers
respond to the programs they are given, allowing also for standard
pseudo-random number generation. While assigning intentionality to
machines can be useful in some situations (Dennett 1987), we consider
this unhelpful in the current context: the successor programs produced
by an AI program are a mechanistic, predictable property of the program
and the instruction set of the underlying computer. In practice, a fixed
amount of time would be available for each AI to produce its successor,
and a run of a program \(m\) could crash or fail to return an answer
within this time limit. We thus define \(Q_{n m}\) as the conditional
probability that program \(m\) returns successor \(n\) under the
condition that it returns any successor at all. If program \(m\) can
never return any successor, \(Q_{n m}\) is undefined.

The fitness function \(f : \Omega \to [0, \infty)\) represents the
amount of computational resource humans choose to allocate to each
program's descendants. The evolutionary theory is agnostic to how this
fitness function is determined, but it is helpful to consider two
examples. In the first, the AIs are also assigned computational tasks
other than designing their descendants, and the fitness function is
objectively determined by how well they perform those jobs.
Alternatively, the decision could be made by human judgment, for example
through conversation with the AIs, by testing the descendants they
produce, or by attempting to gauge alignment. Note that if a run of
program \(m\) fails to complete within the time limit, then no successor
can be assigned, so \(f\) will be lower for programs that frequently
fail. Furthermore, \(f_m\) must be zero for any program \(m\) that can
never return a successor; this means that the non-definition of
\(Q_{n m}\) for such \(m\) presents no problem.

As in the biological case, we define a fitness operator \(\mathbf{F}\)
as a diagonal operator with entries \(F_{m m} = f_m\), and an evolution
operator \(\mathbf{A} = \mathbf{Q}\mathbf{F}\). We define an
unnormalized abundance vector, which evolves according to
\(\mathbf{y}(t) = \mathbf{A}^t \mathbf{y}(0)\). We define the normalized
abundance vector as \[
\mathbf{x}(t)=\frac{\mathbf{y}(t)}{\|\mathbf{y}(t)\|_1}.
\]

The primary difference to the biological model is that the operators
\(\mathbf{Q}\) and \(\mathbf{A}\) are infinite-dimensional and strongly
directed due to the underlying tree structure. One can never return to
the same state by repeated application of \(\mathbf{A}\); formally,
\((\mathbf{A}^t)_{m m}=0\), for all \(t\) and \(m\). We consider the
population as starting on a single root program \(o\):
\(\mathbf{y}(0)=\mathbf{e}_o\); the case of multiple initial programs
can be handled by assigning a virtual root with one descendant for each
initial program. Because of the tree structure, any program \(n\) can
only be produced at one possible time, which we denote by \(|n|\).

A given program might never be produced by the process of successive
self-design starting from the root program \(o\). We say a program \(m\)
is \textbf{reachable} if it is eventually produced, i.e.~if there is a
\(t\) such that \((\mathbf{A}^t)_{m o}>0\).

\section{Lineage analysis}\label{sec-lineage-analysis}

We begin our analysis of the model by introducing \emph{lineage
exponents}: numbers characterizing not just a program's immediate
fitness, but the fitness expected of its future descendants. We show how
lineage exponents can be used to characterize the future success of
traits or programs, introducing the concepts of takeover, survival, and
extinction.

\subsection{Traits and population
share}\label{traits-and-population-share}

We define a \textbf{trait} to be a binary property of programs: each
program either has the trait or not. We can thus formalize a trait as a
subset \(T\subseteq \Omega\). For such a trait, its \textbf{population
share} at time \(t\) is the fraction of the population that has the
trait: \[
\pi_T(t):=\sum_{n\in T} x_n(t).
\]

An example of a trait is being a descendant of a particular program
\(n\): a program \(m\) has the trait \(T_n\) if it is a descendant of
\(n\), i.e.~\((\mathbf{A}^s)_{mn}>0\) for some \(s\ge 0\). We call this
trait the \emph{lineage} of \(n\), and its population share
\(\pi_{T_n}(t)\) is the fraction of the population that is a descendant
of \(n\) at time \(t\).

We call a trait \textbf{heritable} (or \textbf{evolutionarily closed})
if whenever \(n\in T\) and \(Q_{mn}>0\), one also has \(m\in T\).
Lineages are the most important examples of heritable traits.

We can define the long-term evolutionary success of a trait in terms of
its limiting population share. If \(\lim_{t\to\infty} \pi_T(t) = 1\) we
say the trait \textbf{takes over} the population. If
\(\lim_{t\to\infty} \pi_T(t) = 0\) we say the trait \textbf{dies out}.

This limit need not exist. Nevertheless, it is always possible to define
the long-run success of a trait using the concepts of limit superior and
limit inferior (\(\limsup\) and \(\liminf\);
Figure~\ref{fig-limsup-liminf-dominance}). The limit superior of a
sequence \(a(t)\), denoted \(\limsup a(t)\), is the largest value
reached or exceeded infinitely often, and the limit inferior
\(\liminf a(t)\) is the smallest value reached or fallen below
infinitely often. \(\limsup\) and \(\liminf\) are always well-defined,
although their values may be infinite in the case of unbounded
sequences. If the ordinary limit \(\lim a(t)\) exists, then all three
limit types are equal: \(\lim a(t) = \limsup a(t) = \liminf a(t)\). But
if the ordinary limit does not exist, then the limit superior is
strictly greater: \(\limsup a(t) > \liminf a(t)\). This happens when
\(a(t)\) oscillates indefinitely without converging to a limit, and then
\(\limsup\) and \(\liminf\) record the asymptotic upper and lower bounds
of this oscillation.

In the case of population share of a trait, the limit superior thus
defines a ceiling on the long-run success of the trait, while the limit
inferior defines a floor. If \(\limsup \pi_T(t) > 0\) we say the trait
\textbf{survives}; this means that its population share cannot
permanently converge to 0, but instead the trait repeatedly, if
sporadically, regains a substantial fraction of the population. If
\(\liminf \pi_T(t) > 0\), a stronger condition, we say the trait
\textbf{prospers}: although its population share may fluctuate, after
some time there is a floor which it never falls below.

\begin{figure}[!t]

\centering{

\pandocbounded{\includegraphics[keepaspectratio]{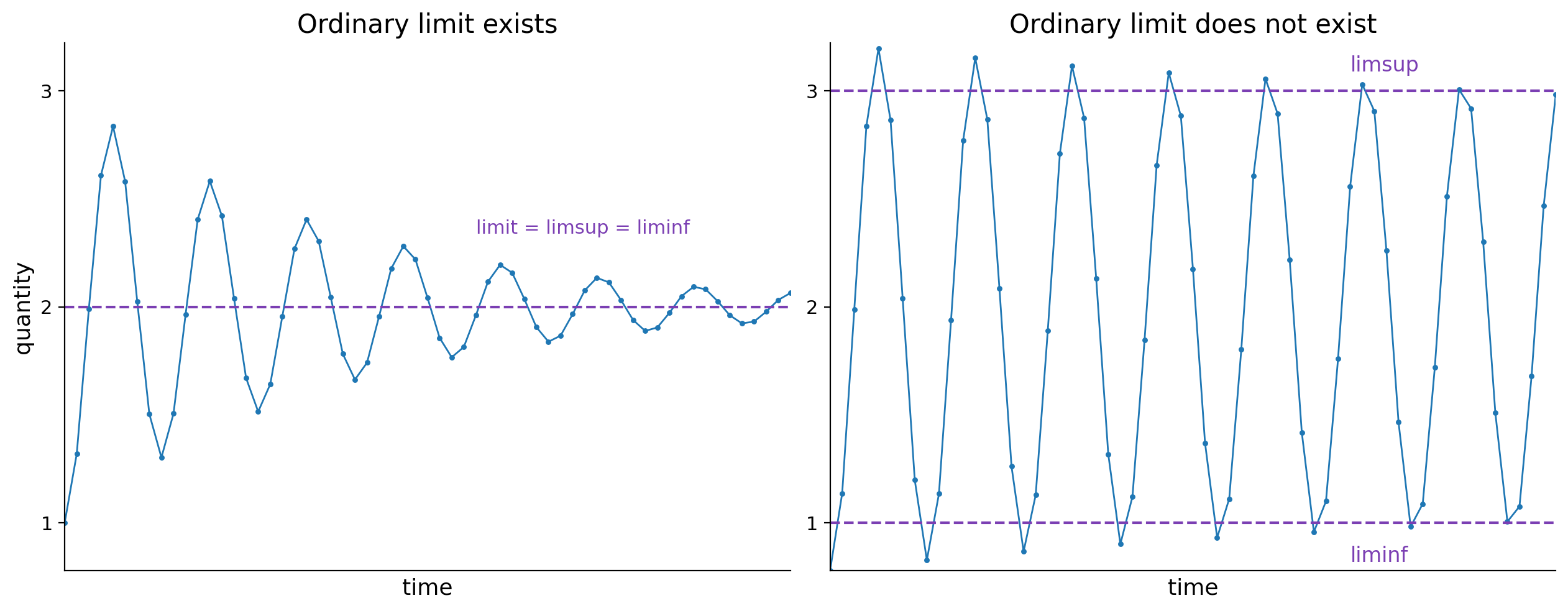}}

}

\caption{\label{fig-limsup-liminf-dominance}Illustration of \(\limsup\)
and \(\liminf\) for an oscillating sequence. Left: an ordinary limit
exists, and \(\limsup\) and \(\liminf\) are equal. Right: an ordinary
limit does not exist, and \(\limsup\) and \(\liminf\) are different.}

\end{figure}%

We use the same terms for program lineages:

\begin{itemize}
\tightlist
\item
  A program \(n\) \textbf{takes over} if \(\pi_{T_n}(t)\to 1\).
\item
  A program \(n\) \textbf{dies out} if \(\pi_{T_n}(t)\to 0\).
\item
  A program \(n\) \textbf{survives} if \(\limsup \pi_{T_n}(t)>0\).
\item
  A program \(n\) \textbf{prospers} if \(\liminf \pi_{T_n}(t)>0\).
\end{itemize}

\subsection{Unnormalized population size, trait size, and lineage
size}\label{unnormalized-population-size-trait-size-and-lineage-size}

The normalized population share \(\mathbf{x}(t)\) quantifies the
evolutionary success of a trait. Nevertheless, we will find it
convenient to study evolutionary dynamics by focusing on unnormalized
population sizes \(\mathbf{y}(t)\).

First let us define some notation. The total unnormalized population
size at time \(t\) is \[
Z_o(t):=\sum_{m\in\Omega} (\mathbf{A}^t)_{mo}.
\]

For a trait \(T\subseteq \Omega\), its unnormalized size at time \(t\)
is \[
Z^{(T)}(t):=\sum_{m\in T} (\mathbf{A}^t)_{mo}.
\] Hence \[
\pi_T(t)=\frac{Z^{(T)}(t)}{Z_o(t)}
\qquad \text{whenever } Z_o(t)>0.
\]

For a program \(n\), define its \textbf{lineage size} after \(s\)
further steps by \[
Z_n(s):=\sum_{m\in\Omega} (\mathbf{A}^s)_{mn}.
\] This is the total unnormalized descendant mass generated by one unit
mass placed at program \(n\).

\subsection{Mean fitness and total population
size}\label{mean-fitness-and-total-population-size}

We are now ready to prove our first result: that total unnormalized
population size evolves multiplicatively, with multiplier the
population's arithmetic mean fitness.

Denote the arithmetic mean fitness at time \(t\) by \[
\langle f(t)\rangle:=\sum_n x_n(t)f_n.
\]

Then we have

\begin{lemma}[Mean fitness determines total population
growth]\protect\hypertarget{lem-total-mass-recursion}{}\label{lem-total-mass-recursion}

For every time \(t\), \[
Z_o(t+1)=\langle f(t)\rangle\, Z_o(t).
\] Hence, since \(Z_o(0)=1\), \[
Z_o(t)=\prod_{s=0}^{t-1}\langle f(s)\rangle.
\]

\end{lemma}

\begin{proof}
Using \(\mathbf{y}(t+1)=\mathbf{Q}\mathbf{F}\mathbf{y}(t)\) and
\(\sum_m Q_{mn}=1\), \[
\begin{aligned}
Z_o(t+1)
&= \sum_m y_m(t+1) \\
&= \sum_m \sum_n Q_{mn} f_n y_n(t) \\
&= \sum_n f_n y_n(t)\sum_m Q_{mn} \\
&= \sum_n f_n y_n(t).
\end{aligned}
\] Since \(y_n(t)=Z_o(t)x_n(t)\), this becomes \[
Z_o(t+1)=Z_o(t)\sum_n x_n(t)f_n
=Z_o(t)\,\langle f(t)\rangle.
\] Iterating from \(Z_o(0)=1\) gives the product formula.
\end{proof}

A similar multiplicative identity holds inside a descendant lineage. For
a program \(n\) and a time \(s\) with \(Z_n(s)>0\), define the mean
fitness inside the lineage descending from \(n\) after \(s\) further
steps by \[
\langle f(s)\rangle_n
:=
\frac{\sum_m f_m (\mathbf{A}^s)_{mn}}{\sum_m (\mathbf{A}^s)_{mn}}
=
\frac{\sum_m f_m (\mathbf{A}^s)_{mn}}{Z_n(s)}.
\]

\begin{lemma}[Lineage mass
recursion]\protect\hypertarget{lem-lineage-mass-recursion}{}\label{lem-lineage-mass-recursion}

For every program \(n\) and every time \(s\) with \(Z_n(s)>0\), \[
Z_n(s+1)=\langle f(s)\rangle_n\, Z_n(s).
\] Hence, since \(Z_n(0)=1\), \[
Z_n(s)=\prod_{t=0}^{s-1}\langle f(t)\rangle_n.
\]

\end{lemma}

\begin{proof}
We obtain a proof by applying Lemma~\ref{lem-total-mass-recursion} to a
shifted process in which the root is program \(n\) rather than \(o\). In
this shifted process, the unnormalized population size at time \(t\) is
exactly \(Z_n(t)\), and the arithmetic mean fitness at time \(t\) is
exactly \(\langle f(t)\rangle_n\).
\end{proof}

\begin{refremark}[Why there is no corresponding formula for an arbitrary
trait]
For a general trait \(T\), there is no analogous identity involving only
the mean fitness of mass already in \(T\), even if \(T\) is heritable,
because mass can enter \(T\) from outside. Indeed, \[
Z^{(T)}(t+1)
=
\sum_{n \in T} f_n y_n(t)\sum_{m \in T} Q_{mn}
+
\sum_{n \notin T} f_n y_n(t)\sum_{m \in T} Q_{mn}.
\] The second term is an immigration term. So the multiplicative product
formulas apply to the whole population and to descendant lineages, but
not to arbitrary traits.

\label{rem-no-arbitrary-trait-product-formula}

\end{refremark}

\subsection{Trait and lineage
exponents}\label{trait-and-lineage-exponents}

The unnormalized trait size \(Z^{(T)}(t)\) is a natural
exponential-scale quantity attached to a trait. Its \(t^{th}\) root
measures the average multiplicative success per generation over \(t\)
steps.

\begin{definition}[Trait
exponent]\protect\hypertarget{def-trait-exponent}{}\label{def-trait-exponent}

For a trait \(T\subseteq \Omega\), if the limit \[
g^{(T)}:=\lim_{t\to\infty} \bigl(Z^{(T)}(t)\bigr)^{1/t}
\] exists, we call it the \textbf{trait exponent} of \(T\).

\end{definition}

When \(T=T_n\) is the descendant lineage of a program \(n\), it is often
more natural to restart the clock at \(n\) itself and work with
\(Z_n(s)\). This gives the corresponding lineage exponent.

\begin{definition}[Lineage
exponent]\protect\hypertarget{def-lineage-exponent}{}\label{def-lineage-exponent}

If the limit \[
g_n := \lim_{s\to\infty} Z_n(s)^{1/s}
\] exists, we call it the \textbf{lineage exponent} of the node \(n\).

\end{definition}

\begin{proposition}[The lineage exponent is the running geometric mean
of lineage mean
fitness]\protect\hypertarget{prp-lineage-exponent-geometric-mean}{}\label{prp-lineage-exponent-geometric-mean}

If the lineage exponent \(g_n\) exists, then \[
g_n = \lim_{s\to\infty}
\left(
\prod_{t=0}^{s-1}\langle f(t)\rangle_n
\right)^{1/s}.
\]

\end{proposition}

\begin{proof}
This is immediate from Lemma~\ref{lem-lineage-mass-recursion}.
\end{proof}

\phantomsection\label{rmk-lineage-exponent-need-not-exist}
\subsection{A lineage exponent need not
exist}\label{a-lineage-exponent-need-not-exist}

\begin{figure}[!t]

\centering{

\pandocbounded{\includegraphics[keepaspectratio]{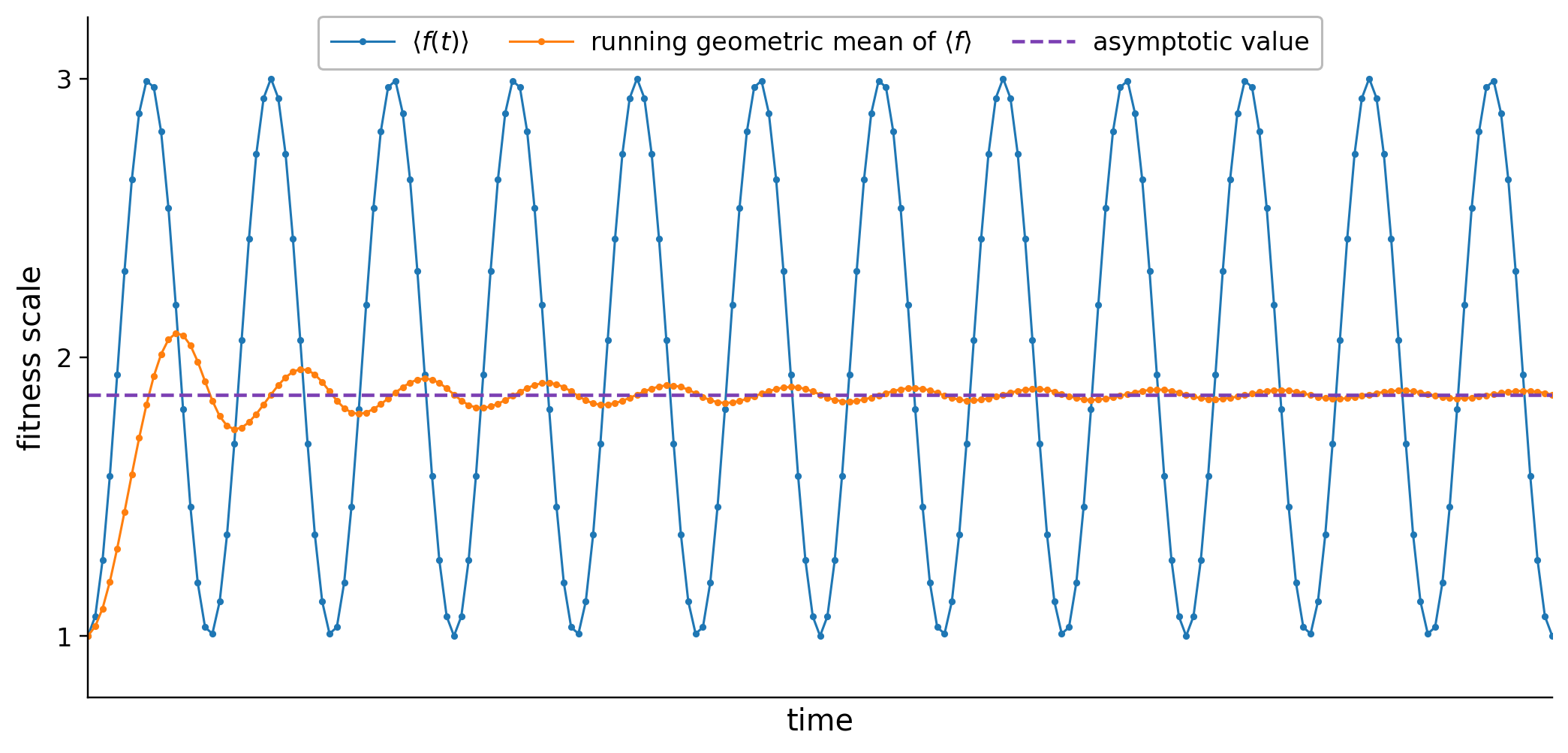}}

}

\caption{\label{fig-running-geometric-mean-fitness}The population mean
fitness \(\langle f \rangle_n\) can oscillate indefinitely, while its
running geometric mean can still settle down. In this case, the ordinary
lineage exponent \(g_n\) exists despite continuing fluctuations in mean
fitness of the lineage.}

\end{figure}%

The lineage exponent can exist even when the population fitness
\(\langle f(t)\rangle\) oscillates, because the running geometric mean
averages multiplicative performance over all earlier generations
(Figure~\ref{fig-running-geometric-mean-fitness}). However, there are
scenarios when even this running geometric mean fails to converge.

For example, consider a simple evolutionary tree consisting of a single
ray (i.e.~a single sequence of programs with no branching),
\(o=n_0 \to n_1 \to n_2 \to \cdots\), with \(Q_{n_{k+1},n_k}=1\) for
every \(k\). Let the fitness sequence alternate between periods of
fitness \(1\) and periods of fitness \(2\), with the length of the
\(k^{th}\) block increasing very rapidly, for example as \(2^{2^k}\).
Then each block is so long that it dominates all earlier ones, and the
running geometric mean oscillates between \(1\) and \(2\) indefinitely.

\begin{definition}[Upper and lower trait
exponents]\protect\hypertarget{def-upper-lower-trait-exponents}{}\label{def-upper-lower-trait-exponents}

To deal with this possibility, we make use of \(\limsup\) and
\(\liminf\). For each trait \(T\subseteq \Omega\), define its
\textbf{upper} and \textbf{lower trait exponents} by \[
\overline g^{(T)}:=\limsup_{t\to\infty} \bigl(Z^{(T)}(t)\bigr)^{1/t},
\qquad
\underline g^{(T)}:=\liminf_{t\to\infty} \bigl(Z^{(T)}(t)\bigr)^{1/t}.
\]

\end{definition}

\begin{definition}[Upper and lower lineage
exponents]\protect\hypertarget{def-upper-lower-lineage-exponents}{}\label{def-upper-lower-lineage-exponents}

For each program \(n\), define its \textbf{upper} and \textbf{lower
lineage exponents} by \[
\overline g_n:=\limsup_{s\to\infty} Z_n(s)^{1/s},
\qquad
\underline g_n:=\liminf_{s\to\infty} Z_n(s)^{1/s}.
\] For the root, these are exactly the asymptotic upper and lower bounds
of the running geometric mean fitness: \[
\overline g_o
=
\limsup_{t\to\infty}
\left(
\prod_{s=0}^{t-1}\langle f(s)\rangle
\right)^{1/t},
\qquad
\underline g_o
=
\liminf_{t\to\infty}
\left(
\prod_{s=0}^{t-1}\langle f(s)\rangle
\right)^{1/t}.
\]

\end{definition}

\begin{proof}
This follows from Lemma~\ref{lem-total-mass-recursion}.
\end{proof}

\subsection{Lineage exponents cannot increase along
descendants}\label{lineage-exponents-cannot-increase-along-descendants}

We now prove a fundamental monotonicity property of lineage exponents:
they cannot increase along descendants. Informally, a program's lineage
exponent represents the best possible long-run fitness of any branch of
its lineage; many individual descendants' lineages will not live up to
this potential, so the lineage exponent can go down, but not up, along
generations.

\begin{theorem}[Upper and lower lineage exponents are non-increasing
along
descendants]\protect\hypertarget{thm-growth-nonincreasing-along-descendants}{}\label{thm-growth-nonincreasing-along-descendants}

If \(m\) is a descendant of \(n\), meaning that
\((\mathbf{A}^k)_{mn}>0\) for some \(k\ge 0\), then \[
\overline g_m \le \overline g_n,
\qquad
\underline g_m \le \underline g_n.
\]

\end{theorem}

\begin{proof}
Because all entries of \(\mathbf{A}\) are nonnegative, \[
(\mathbf{A}^{t+k})_{\ell n}
=
\sum_j (\mathbf{A}^t)_{\ell j}(\mathbf{A}^k)_{jn}
\ge
(\mathbf{A}^t)_{\ell m}(\mathbf{A}^k)_{mn}.
\] Summing over \(\ell\) gives \[
Z_n(t+k)\ge (\mathbf{A}^k)_{mn} Z_m(t).
\] Multiplicative constants and finite time shifts do not affect
\(\limsup\) or \(\liminf\) of \(t^{th}\) roots, so the stated
inequalities follow.
\end{proof}

\subsection{Criteria for takeover, extinction, and
survival}\label{criteria-for-takeover-extinction-and-survival}

Trait exponents can provide information about takeover, survival, and
extinction.

\begin{theorem}[Takeover
criterion]\protect\hypertarget{thm-takeover-criterion}{}\label{thm-takeover-criterion}

Let \(\Omega=S\sqcup T\) be a partition into two complementary traits.
If \[
\underline g^{(S)}>\overline g^{(T)},
\] then \(S\) takes over and \(T\) dies out: \[
\pi_S(t)\to 1, \quad \pi_T(t)\to 0.
\]

\end{theorem}

\begin{proof}
Choose numbers \(a,b\) with \[
\overline g^{(T)}<a<b<\underline g^{(S)}.
\] Then for all sufficiently large \(t\), \[
Z^{(T)}(t)\le a^t,
\qquad
Z^{(S)}(t)\ge b^t
\] Therefore \[
\pi_T(t)
=
\frac{Z^{(T)}(t)}{Z^{(S)}(t)+Z^{(T)}(t)}
\le
\frac{a^t}{b^t+a^t}
=
\frac{1}{1+(b/a)^t}.
\] Hence \(\pi_S(t)=1-\pi_T(t)\to 1\).
\end{proof}

\begin{theorem}[Extinction
criterion]\protect\hypertarget{thm-extinction-criterion}{}\label{thm-extinction-criterion}

If a trait \(T\) satisfies \[
\overline g^{(T)} < \underline g_o,
\] then \(T\) dies out: \[
\pi_T(t)\to 0.
\]

\end{theorem}

\begin{proof}
Choose \(a,b\) with \[
\overline g^{(T)}<a<b<\underline g_o.
\] Then for large \(t\), \[
Z^{(T)}(t)\le a^t,
\qquad
Z_o(t)\ge b^t
\]

Hence

\[
\pi_T(t)=\frac{Z^{(T)}(t)}{Z_o(t)}\le \frac{a^t}{b^t}\to 0
\]

Therefore \(\limsup_{t\to\infty}\pi_T(t)=0\).
\end{proof}

Note that comparing exponents to the root can tell us that a trait dies
out, but not that it takes over. By
Theorem~\ref{thm-growth-nonincreasing-along-descendants},
\(\overline g^{(T)} \leq \overline g_o\). But even if
\(\overline g^{(T)} = \overline g_o\) the trait might still die out.
This is because the lineage exponents only capture the exponential
scale, and other factors may act as a ``tiebreaker''. For example,
consider a simple tree with two rays emerging from the root, one ray
with constant fitness \(1\), and the other ray with fitness \(t/(t+1)\)
at time \(t\). Both rays have lineage exponent \(1\), but the second ray
dies out: its unnormalized population size at time \(t\) is \(1/t\), so
its population share tends to \(0\).

\begin{theorem}[Survival
criterion]\protect\hypertarget{thm-survival-criterion}{}\label{thm-survival-criterion}

Let \(\Omega=S\sqcup T\) be a partition into two complementary traits.
If \[
\overline g^{(S)}>\overline g^{(T)},
\] then \(S\) survives, and furthermore \(\limsup \pi_S(t)=1\).

\end{theorem}

\begin{proof}
Choose numbers \(a,b\) with \[
\overline g^{(T)}<a<b<\overline g^{(S)}.
\] By the definition of \(\overline g^{(T)}\), for all sufficiently
large \(t\), \[
Z^{(T)}(t)\le a^t.
\] By the definition of \(\overline g^{(S)}\), there exist infinitely
many times \(t\) such that \[
Z^{(S)}(t)\ge b^t.
\] Hence for infinitely many arbitrarily large \(t\), \[
\pi_T(t)
=
\frac{Z^{(T)}(t)}{Z^{(S)}(t)+Z^{(T)}(t)}
\le
\frac{a^t}{b^t+a^t}
=
\frac{1}{1+(b/a)^t}.
\] Since \(b>a\), the right-hand side tends to \(0\) along those times.
Therefore \[
\pi_S(t)=1-\pi_T(t)\to 1
\] along an infinite subsequence. This proves \[
\limsup_{t\to\infty}\pi_S(t)=1.
\] In particular, \(S\) survives.
\end{proof}

Thus, if \(\overline g^{(S)}>\overline g^{(T)}\), then not only does
\(S\) survive, there are infinitely many moments where it almost takes
over the population. However, unless also
\(\underline g^{(S)}>\overline g^{(T)}\), we cannot conclude that \(S\)
actually takes over, because \(T\) might make repeated comebacks.

\subsection{Winnowing when all lineage exponents
exist}\label{winnowing-when-all-lineage-exponents-exist}

If it happens that for every reachable node \(n\) the ordinary lineage
exponent \(g_n := \lim_{t\to\infty} Z_n(t)^{1/t}\) exists, then
evolution admits a simple ``winnowing'' interpretation.

By monotonicity along descendants, if \(m \preceq n\) then
\(g_n \le g_m\). In particular, every reachable node satisfies
\(g_n \le g_o\). So the root exponent \(g_o\) is the largest that can
occur anywhere in the tree.

Now let \(n\) be a reachable node with strictly smaller exponent than
the root: \(g_n < g_o\). By the extinction criterion, the normalized
share of the descendants of \(n\) must vanish: \(\pi_{T_n}(t) \to 0\).
Thus, a branch can contribute to the long-term surviving population only
if it preserves the root's exponent.

This gives a useful picture of the dynamics. At each generation, any
program whose lineage exponent has already dropped below \(g_o\) may
still produce many descendants for a while, but all of those descendants
are transient in normalized terms. They are eventually winnowed away by
branches that continue to realize the larger exponent \(g_o\). The only
programs that can have long-term surviving descendants are those with
\(g_n = g_o\). If furthermore each generation has a single program with
the largest lineage exponent, then this one program will dominate: far
enough into the future, all programs will be descendants of this
ancestor.

\begin{example}[Binary
tree]\protect\hypertarget{exm-binary-tree}{}\label{exm-binary-tree}

To illustrate a case where all lineage exponents exist, consider an
example of a binary tree. Each program can give rise to two children,
with \(\mathbf{Q}\) specifying a 50\% probability of each. A program
\(n\) at time \(t\) can thus be specified by a binary string
\(b_1\cdots b_t\), summarizing which descendant was followed at each
previous step. We model its fitness as \[
f_n
:=
1+\sum_{j=1}^t (2b_j-1) 2^{-j}.
\] Thus the root program has fitness 1, and fitness always lies in
\((0,2)\). Each reproduction step changes fitness by \(\pm 2^{-(t+1)}\):
one child is slightly less fit than its parent, and the other is
slightly more fit.

\begin{figure}[!t]

\centering{

\pandocbounded{\includegraphics[keepaspectratio]{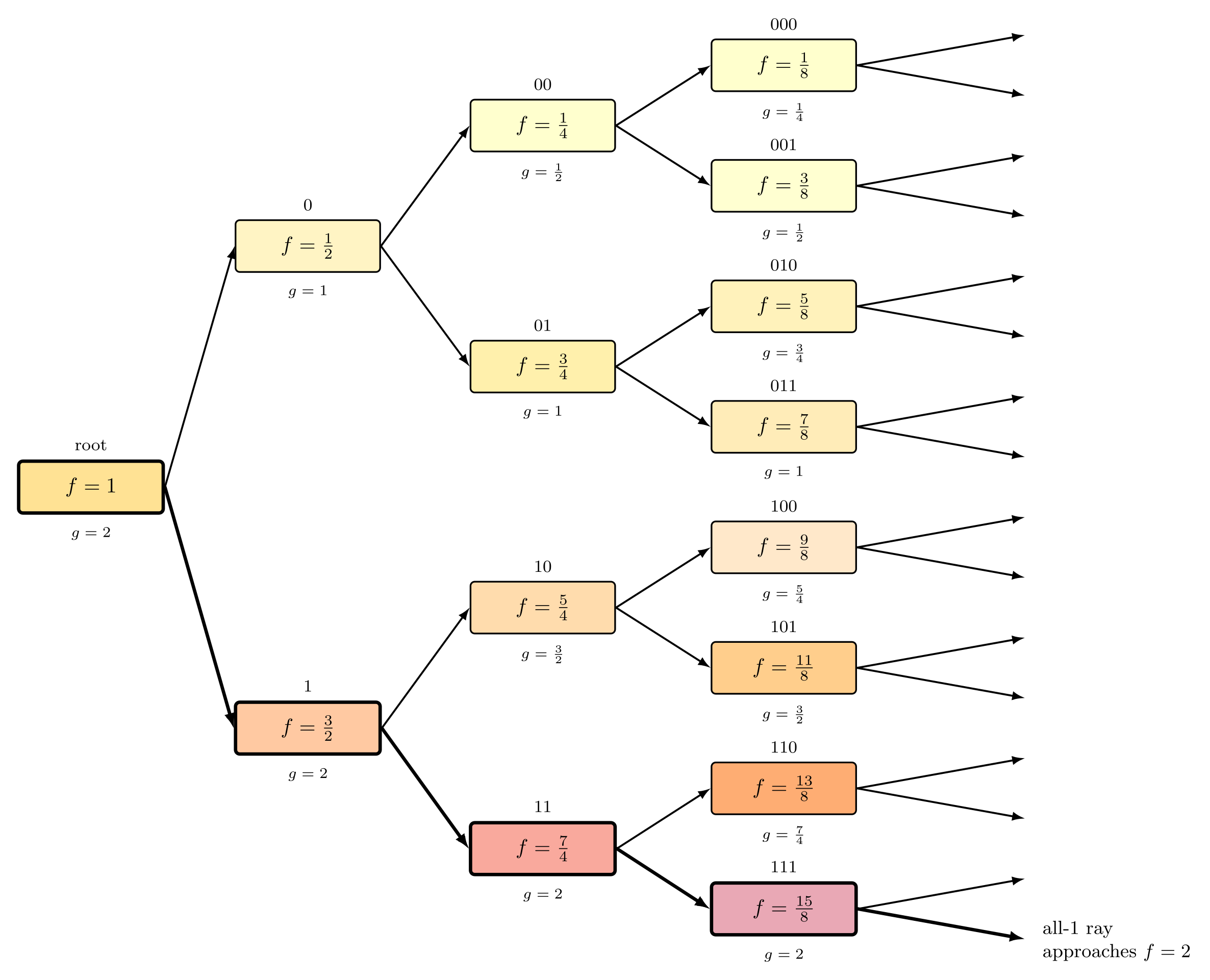}}

}

\caption{\label{fig-dyadic-tree-example}The binary tree example. Each
program has two possible children, whose fitness either increases or
decreases by \(2^{-(t+1)}\). Each program can be identified by a binary
string (above the box). The optimal fitness of \(2\) is not achieved by
any program, while the all-\(1\) ray approaches the unattained supremum
\(2\). Shading represents normalized population share of each program;
thick boxes indicate programs with the optimal lineage exponent of
\(g=2\). These are the only programs whose descendants will have
long-term survival.}

\end{figure}%

Every descendant of this program is obtained by appending another binary
string \(u_1\cdots u_r\), and has fitness \[
f_{n, u_1\cdots u_r}
=
1 + \sum_{i=1}^t (2b_i-1) 2^{-i} + \sum_{j=1}^r (2u_j-1) 2^{-(t+j)}
= f_n + \sum_{j=1}^r (2u_j-1) 2^{-(t+j)}.
\] Hence all descendants of \(n\) have fitness in the interval
\((f_n-2^{-t},f_n+2^{-t})\), and the supremum of \(n\)'s descendant
fitnesses is \(M_n := f_n+2^{-t}\). This value is not attained by any
finite descendant, but it is approached along the all-\(1\) continuation
of \(n\). We next compute the lineage exponents in this example.

\end{example}

\begin{proposition}[Binary tree lineage
exponents]\protect\hypertarget{prp-lineage-exponent-binary-tree}{}\label{prp-lineage-exponent-binary-tree}

For every program \(n=b_1\cdots b_t\), the lineage exponent exists and
is given by the supremum of its descendant fitnesses, \[
g_n := f_n+2^{-t}
\]

\end{proposition}

\begin{proof}
Starting from one unit of mass at \(n\), the total descendant mass after
\(s\) generations is obtained by summing over the \(2^s\) possible
descendants after that time: \[
Z_n(s)
=
2^{-s}
\sum_{u_1,\ldots,u_s\in\{0,1\}}
\prod_{r=0}^{s-1}
\left(f_n + \sum_{i=1}^r (2u_i-1) 2^{-(t+i)}\right),
\] where the inner sum is interpreted as \(0\) when \(r=0\).

For every path and every factor in the product, \[
f_n + \sum_{i=1}^r (2u_i-1) 2^{-(t+i)} \le f_n + \sum_{i=1}^\infty 2^{-(t+i)} = M_n.
\] Therefore every summand is at most \(2^{-s} M_n^s\), and since there
are \(2^s\) summands, \(Z_n(s) \le M_n^s\). It follows that \[
\overline g_n \le M_n.
\]

For the reverse inequality, fix \(m\ge 1\) and restrict attention to
those depth-\(s\) descendants whose first \(m\) appended bits are all
\(1\). There are \(2^{s-m}\) such descendants. Along any such path,
after those first \(m\) steps have been taken, every later fitness
factor is at least \[
f_n + \sum_{i=1}^m 2^{-(t+i)} - \sum_{i=m+1}^\infty 2^{-(t+i)} = M_n - 2^{-(t+m-1)}.
\] Hence there is a constant \[
C_m := \prod_{r=0}^{m-1}
\left(f_n+\sum_{i=1}^r 2^{-(t+i)}\right) > 0
\] such that every one of these selected descendants contributes at
least \[
2^{-s} C_m \bigl(M_n-2^{-(t+m-1)}\bigr)^{s-m}
\] to \(Z_n(s)\). Summing over the \(2^{s-m}\) such descendants gives \[
Z_n(s)
\ge
2^{-m} C_m \bigl(M_n-2^{-(t+m-1)}\bigr)^{s-m}.
\] Taking \(s^{th}\) roots and then letting \(s\to\infty\) yields \[
\underline g_n \ge M_n-2^{-(t+m-1)}.
\] Since this holds for every \(m\), letting \(m\to\infty\) gives
\(\underline g_n \ge M_n\). Combining this with the upper bound, \[
\underline g_n = \overline g_n = M_n.
\] This proves the claim.
\end{proof}

This example illustrates the behavior we find when all lineage exponents
exist. The root's lineage exponent is \(2\): a fitness value which is
never achieved, but is approached arbitrarily closely by lineages whose
initial segment consists entirely of \(1\)s. Any program whose binary
representation is not all \(1\)s has a lineage exponent below that of
the root, and its lineage dies out. We thus observe a progressive
winnowing of the programs, leaving only the all-\(1\) lineage.

\section{\texorpdfstring{\(\eta\)-preservation}{\textbackslash eta-preservation}}\label{sec-eta-preservation}

We next introduce a simple condition that guarantees a lower bound on
the running geometric mean fitness of a lineage, and thus on the
long-run growth of the population. This condition requires that every
reproducing program has a non-negligible probability of producing an
offspring whose fitness is at least as good as its own. One way to do
this is for a program to suggest itself as a descendant with probability
\(\eta\). This condition is \emph{not} strong enough to ensure
convergence of population fitness to its maximum possible value; we will
describe a criterion that is strong enough to do that in
Section~\ref{sec-eta-locking}, but will first analyze the weaker
\(\eta\)-preservation condition.

\begin{definition}[\(\eta\)-preservation]\protect\hypertarget{def-eta-preservation}{}\label{def-eta-preservation}

Fix \(\eta>0\). We say that the system satisfies the
\(\eta\)-preservation condition if for every program \(n\), \[
\sum_{m:\, f_m \ge f_n} Q_{mn} \ge \eta.
\] Thus every reproducing program has probability at least \(\eta\) of
producing a descendant whose fitness is at least its own.

\end{definition}

\begin{proposition}[\(\eta\)-preservation gives a programwise lower
bound on the lower lineage
exponent]\protect\hypertarget{prp-nodewise-lower-growth-bound}{}\label{prp-nodewise-lower-growth-bound}

Assume the \(\eta\)-preservation condition. Then every reachable program
\(n\) satisfies \[
\underline g_n \ge \eta f_n.
\]

\end{proposition}

\begin{proof}
Fix a reachable program \(n\). Start with one unit mass at \(n\) and
evolve it by \[
\mathbf{y}^{(n)}(0)=e_n,
\qquad
\mathbf{y}^{(n)}(t+1)=\mathbf{A}\mathbf{y}^{(n)}(t).
\] Thus \[
Z_n(t)=\|\mathbf{y}^{(n)}(t)\|_1.
\]

Define a tracked subpopulation \(\mathbf{w}(t)\le \mathbf{y}^{(n)}(t)\)
by keeping only those offspring steps that do not decrease fitness: let
\(\mathbf{w}(0)=e_n\), and recursively set \[
w_m(t+1):=\sum_j Q_{mj} f_j w_j(t)\mathbf{1}_{\{f_m \ge f_j\}}.
\] Then \(0 \le w_m(t)\le y^{(n)}_m(t)\) for all \(m,t\). Write \[
W(t):=\|\mathbf{w}(t)\|_1.
\] Then \(W(t)\le Z_n(t)\) for every \(t\).

Now, every program in the tracked subpopulation has fitness at least
\(f_n\). So \[
\begin{aligned}
W(t+1)
&=
\sum_m w_m(t+1) \\
&=
\sum_m \sum_j Q_{mj} f_j w_j(t)\mathbf{1}_{\{f_m \ge f_j\}} \\
&=
\sum_j f_j w_j(t)\sum_{m:\, f_m \ge f_j} Q_{mj}.
\end{aligned}
\] Because we have assumed \(\eta\)-preservation,
\(\sum_{m:\, f_m \ge f_j} Q_{mj} \ge \eta\), so \[
W(t+1)\ge \eta \sum_j f_j w_j(t).
\] Since \(w_j(t)\) is supported on programs with \(f_j \ge f_n\), we
have \[
\sum_j f_j w_j(t)\ge f_n \sum_j w_j(t)=f_n W(t).
\] Therefore \[
W(t+1)\ge \eta f_n W(t).
\] Starting from \(W(0)=1\), induction gives \[
W(t)\ge (\eta f_n)^t.
\] Because \(Z_n(t)\ge W(t)\), taking \(t^{th}\) roots and then
\(\liminf\) yields \(\underline g_n \ge \eta f_n\).
\end{proof}

\begin{corollary}[The root lower lineage exponent is at least
\(\eta f^*\)]\protect\hypertarget{cor-root-lower-growth-bound}{}\label{cor-root-lower-growth-bound}

Assume the \(\eta\)-preservation condition, and let \[
f^* := \sup\{f_n : n \text{ is reachable}\} < \infty.
\] Then \[
\underline g_o \ge \eta f^*.
\]

\end{corollary}

\begin{proof}
Let \(\epsilon>0\). By definition of \(f^*\), there exists a reachable
program \(n\) with \(f_n>f^*-\epsilon\).
Proposition~\ref{prp-nodewise-lower-growth-bound} gives \[
\underline g_n \ge \eta f_n \ge \eta(f^*-\epsilon).
\] Since \(n\) is a descendant of the root,
Theorem~\ref{thm-growth-nonincreasing-along-descendants} gives
\(\underline g_n \le \underline g_o\). Hence
\(\underline g_o \ge \eta(f^*-\epsilon)\). Letting
\(\epsilon \downarrow 0\) proves the claim.
\end{proof}

\begin{theorem}[\(\eta\)-preservation gives a lower bound on the running
geometric mean
fitness]\protect\hypertarget{thm-running-geometric-mean-lower-bound}{}\label{thm-running-geometric-mean-lower-bound}

Assume the \(\eta\)-preservation condition and \(f^*<\infty\). Then \[
\liminf_{t\to\infty}
\left(
\prod_{s=0}^{t-1} \langle f(s)\rangle
\right)^{1/t}
\ge
\eta f^*.
\] Equivalently, \[
\underline g_o \ge \eta f^*.
\]

\end{theorem}

\begin{proof}
The equivalence is Lemma~\ref{lem-total-mass-recursion} together with
the definition of \(\underline g_o\). The bound then follows from
Corollary~\ref{cor-root-lower-growth-bound}.
\end{proof}

\begin{example}[\(\eta\)-preservation need not force convergence of mean
fitness]\protect\hypertarget{exm-preservation-no-mean-convergence}{}\label{exm-preservation-no-mean-convergence}

The lower bound above concerns the running geometric mean. It does
\textbf{not} imply that the arithmetic mean fitness
\(\langle f(t)\rangle\) converges. A simple counterexample is obtained
by combining a persistent spine with alternating side lineages.

Fix \(0<\eta<1\) and choose \(b\in(0,1)\). Build a tree with a main
spine \[
n_0\to n_1\to n_2\to \cdots,
\qquad
f_{n_t}=1
\quad \text{for all } t,
\] and from each spine program \(n_t\) create one auxiliary burst
lineage:

\begin{itemize}
\tightlist
\item
  with probability \(\eta\), offspring go to the next spine program
  \(n_{t+1}\);
\item
  with probability \(1-\eta\), offspring go to a burst program of
  fitness \[
  q_t=
  \begin{cases}
  0, & t \text{ even},\\
  b, & t \text{ odd}.
  \end{cases}
  \]
\end{itemize}

From each burst program of fitness \(q\in\{0,b\}\), give probability
\(\eta\) to a child of the same fitness \(q\) and probability \(1-\eta\)
to a program of fitness \(0\), which will reproduce no further.

This system satisfies \(\eta\)-preservation: every program has
probability at least \(\eta\) of producing a child of fitness at least
its own.

We can analyze its evolutionary dynamics exactly. At generation \(t\),
there will be an unnormalized mass of \(\eta^t\) on the main spine.
Denote the total unnormalized mass on all fitness-\(b\) rays as \(B_t\).
At time \(t+1\), the mass on the existing \(b\)-rays has been scaled by
a factor of \(\eta b\), while injection from the spine contributes
\((1-\eta) \eta^t\) on odd timesteps. Thus

\[
B_{t+1} = \eta b B_t + (1-\eta) \eta^t \mathbf{1}_{\{t\ \text{odd}\}}
\]

If we define \(r_t = B_t/\eta^t\) to be the ratio of mass on the
\(b\)-rays compared to the spine at time \(t\), then we have a recursion
for even times:

\[r_{2k+2} = b^2 r_{2k} + (1-\eta)/\eta\]

This converges to a limit, \[
r_{2k}\to R = \frac{(1-\eta) }{(1-b^2)\eta}.
\] However on odd times, we obtain a different limit,
\(r_{2k+1} \to b\ R\).

There is also a mass \(C_t\) on zero-fitness nodes at time \(t\), given
by \[
C_{t+1} = (1-\eta)b B_t + (1-\eta)\eta^t \mathbf{1}_{t\text{even}}
\] The mean population fitness thus tends to limits on odd and even
trials:

\[
\langle f \rangle_\text{odd} = \frac{1+b^2 R}{1+bR + (1-\eta)(1+bR)/\eta}
\qquad
\langle f \rangle_\text{even} = \frac{1+b R}{1+R + (1-\eta)b^2 R/\eta}
\]

and we have a situation where the population fitness does not converge,
even though its running geometric mean does.

\end{example}

\subsection{\texorpdfstring{Unbounded reachable fitness under
\(\eta\)-preservation}{Unbounded reachable fitness under \textbackslash eta-preservation}}\label{unbounded-reachable-fitness-under-eta-preservation}

The \(\eta\)-preservation theorem also has a simple consequence when
reachable fitness is unbounded. In that case, the total unnormalized
population grows faster than any exponential rate. Equivalently, the
running geometric mean of arithmetic mean fitness diverges.

\begin{corollary}[Unbounded reachable fitness implies super-exponential
growth]\protect\hypertarget{cor-eta-preservation-unbounded}{}\label{cor-eta-preservation-unbounded}

Assume \(\eta\)-preservation, and suppose reachable fitness is unbounded
in the sense that for every \(M>0\) there exists a reachable node \(n\)
with \(f_n>M\). Then \[
\underline g_o=\infty.
\] Equivalently, \[
\liminf_{t\to\infty} Z_o(t)^{1/t}=\infty.
\] Since \[
Z_o(t+1)=\langle f(t)\rangle Z_o(t),
\] this is the same as \[
\liminf_{t\to\infty}\left(\prod_{s=0}^{t-1}\langle f(s)\rangle\right)^{1/t}=\infty.
\] In particular, \[
\limsup_{t\to\infty}\langle f(t)\rangle=\infty.
\]

\end{corollary}

\begin{proof}
Fix \(M>0\). Since reachable fitness is unbounded, there exists a
reachable node \(n\) with \[
f_n>\frac{M}{\eta}.
\] By Proposition~\ref{prp-nodewise-lower-growth-bound}, every reachable
node satisfies \[
\underline g_n\ge \eta f_n,
\] so in particular \[
\underline g_n>M.
\]

Because \(n\) is reachable, there exists some \(s\ge 0\) such that \[
(\mathbf{A}^s)_{no}>0.
\] For every \(t\ge 0\) we therefore have \[
Z_o(t+s)
=
\sum_{m\in\Omega}(\mathbf{A}^{t+s})_{mo}
=
\sum_{m\in\Omega}\sum_{k\in\Omega}(\mathbf{A}^t)_{mk}(\mathbf{A}^s)_{ko}
\ge
(\mathbf{A}^s)_{no}\sum_{m\in\Omega}(\mathbf{A}^t)_{mn}
=
(\mathbf{A}^s)_{no}\, Z_n(t).
\] Taking \((t+s)\)-th roots and then \(\liminf\) gives \[
\underline g_o\ge \underline g_n>M.
\] Since \(M>0\) was arbitrary, it follows that \[
\underline g_o=\infty.
\]

Finally, if \(\langle f(t)\rangle\) were bounded above by some finite
constant \(B\) for all sufficiently large \(t\), then the identity \[
Z_o(t)=\prod_{s=0}^{t-1}\langle f(s)\rangle
\] would imply \(Z_o(t)^{1/t}\le B+o(1)\), contradicting
\(\underline g_o=\infty\). Hence \[
\limsup_{t\to\infty}\langle f(t)\rangle=\infty.
\]
\end{proof}

\section{\texorpdfstring{\(\eta\)-locking}{\textbackslash eta-locking}}\label{sec-eta-locking}

\begin{figure}[!t]

\centering{

\pandocbounded{\includegraphics[keepaspectratio]{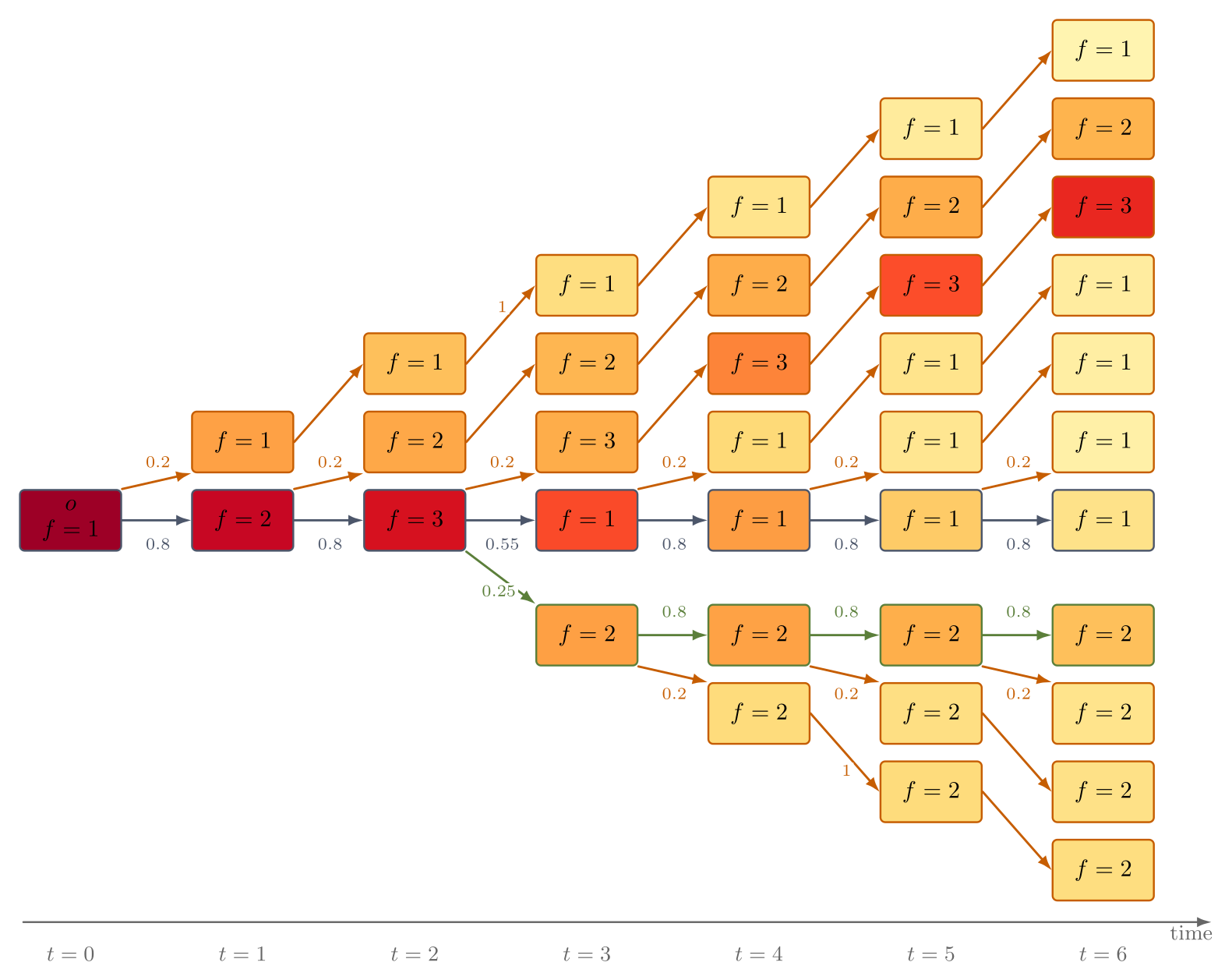}}

}

\caption{\label{fig-eta-locking}\(\eta\)-locking. Under the
\(\eta\)-locking scheme, each program suggests with probability at least
\(\eta\) a ``locked'' copy of itself, with the same fitness, but that
only ever clones itself as a descendant. As a consequence, even if a
program's children have lower fitness, a copy of the original is
retained.}

\end{figure}%

The \(\eta\)-preservation condition guarantees a lower bound on the
running geometric mean fitness, but it does not guarantee convergence of
mean fitness, which can continue to oscillate even with
\(\eta\)-preservation. We now introduce a stronger condition, which
guarantees convergence of fitness to the maximum reachable value. This
condition, known as \(\eta\)-locking (Figure~\ref{fig-eta-locking}),
requires that every reproducing program has a non-negligible probability
of suggesting an offspring that is a ``locked copy'' of itself, which
will only ever produce further locked copies of itself, all with the
same fitness. This condition is stronger than \(\eta\)-preservation
because it creates a heritable reservoir of copies that preserve
whatever fitness has already been achieved.

\subsection{\texorpdfstring{Locked rays and uniform
\(\eta\)-locking}{Locked rays and uniform \textbackslash eta-locking}}\label{locked-rays-and-uniform-eta-locking}

\begin{definition}[Locked
ray]\protect\hypertarget{def-locked-trait}{}\label{def-locked-trait}

A program is said to have a \textbf{locked ray} if it gives rise to an
evolutionarily closed ray of copies of itself, all with the same
fitness, and each with probability \(1\) of producing the next program
on that ray.

\end{definition}

\begin{definition}[\(\eta\)-locking]\protect\hypertarget{def-uniform-eta-locking}{}\label{def-uniform-eta-locking}

We say that \(\eta\)-locking holds if there exists a fixed \(\eta>0\)
such that every program \(n\) has probability at least \(\eta\) of
giving rise to a locked copy of itself.

Thus for every program \(n\) there is a locked ray \(R_n\) beginning
from a child of \(n\) such that

\begin{itemize}
\tightlist
\item
  the first step from \(n\) into \(R_n\) has probability at least
  \(\eta\);
\item
  every later step along \(R_n\) has probability \(1\);
\item
  every program on \(R_n\) has fitness equal to \(f_n\).
\end{itemize}

Let \(\mathcal{L}\) denote the locked trait, i.e.~the union of all
locked rays, and let \(O:=\Omega\setminus \mathcal{L}\) be the
non-locked trait, i.e.~the remainder of the tree.

\end{definition}

\begin{proposition}[Under \(\eta\)-locking, the locked trait always
prospers]\protect\hypertarget{prp-locked-trait-prospers}{}\label{prp-locked-trait-prospers}

Assume \(\eta\)-locking, and let \(\mathcal{L}\) be the locked trait.
Then for every \(t\ge 1\), \[
\pi_{\mathcal{L}}(t)\ge \eta.
\] In particular, the locked trait prospers
(\(\liminf \pi_{\mathcal{L}} > 0\)).

\end{proposition}

\begin{proof}
By the definition of \(\eta\)-locking, each node creates a fraction
\(\eta\) of locked programs on every iteration. Thus the total fraction
cannot be below \(\eta\).
\end{proof}

\begin{lemma}[A reachable locked ray has lineage exponent equal to its
fitness]\protect\hypertarget{lem-reachable-locked-ray-growth}{}\label{lem-reachable-locked-ray-growth}

Let \(n\) be reachable, and let \(R_n\) be its locked ray. Then \[
\underline g^{(R_n)} = \overline g^{(R_n)} = f_n.
\] where \(\underline g^{(R_n)}\) and \(\overline g^{(R_n)}\) represent
the lower and upper trait exponent of \(R_n\), i.e.\\
\(\underline g^{(R_n)} = \liminf \left( \sum_{m\in R_n}(A^t)_{mo} \right)^{1/t}\).

\end{lemma}

\begin{proof}
Because \(n\) is reachable, there exists \(k\ge 0\) such that
\((\mathbf{A}^k)_{no}>0\). Let \(r_1\) be the first program on the
locked ray \(R_n\). Since the transition from \(n\) to \(r_1\) has
probability at least \(\eta\), a positive mass
\(c := Q_{r_1 n} f_n (\mathbf{A}^k)_{no} > 0\) reaches \(r_1\) at time
\(k+1\). After that, all mass on \(R_n\) remains on \(R_n\), every
transition along the ray has probability \(1\), and every program on the
ray has fitness \(f_n\). Hence for every \(t\ge k+1\), \[
Z^{(R_n)}(t) = c\, f_n^{\, t-k-1}.
\] Taking \(t\)-th roots and letting \(t\to\infty\) gives \[
\underline g^{(R_n)} = \overline g^{(R_n)} = f_n.
\]
\end{proof}

\subsection{\texorpdfstring{Convergence under \(\eta\)-locking with
bounded reachable
fitness}{Convergence under \textbackslash eta-locking with bounded reachable fitness}}\label{convergence-under-eta-locking-with-bounded-reachable-fitness}

The analysis of \(\eta\)-locking is simpler when reachable fitness is
bounded, so we start with that case. Let the maximum reachable fitness
be \[
f^* := \sup\{f_n : n \text{ is reachable}\} < \infty.
\]

We now prove that under \(\eta\)-locking, the fitness converges to
\(f^*\). First we show that the arithmetic mean fitness
\(\langle f(t) \rangle \to f^*\). Then we show that almost all
population mass eventually lies on programs whose fitness is arbitrarily
close to \(f^*\).

The proof strategy is simple. For any threshold \(a<f^*\) but still
above \((1-\eta)f^*\), the programs with fitness \(>a\) on locked rays
form a heritable high-fitness set \(H_a\). Some reachable locked ray in
that set grows faster than \(a^t\), while the complement cannot grow
faster than \(a^t\). Hence the high-fitness locked set takes over. We
start with a lemma that establishes the upper bound on growth outside of
\(H_a\).

\begin{lemma}[The low-fitness complement has upper exponent at most
\(a\)]\protect\hypertarget{lem-low-fitness-complement}{}\label{lem-low-fitness-complement}

Fix a number \(a\) such that \((1-\eta)f^* < a < f^*\), and let \(H_a\)
be the set of all programs on locked rays of fitness \(>a\), and \(U_a\)
its complement. Thus,

\[
H_a := \mathcal{L} \cap \{n\in\Omega : f_n>a\},
\qquad
U_a := \Omega \setminus H_a.
\] Then \(H_a\) is a union of locked rays, hence a heritable trait, and
\[
\overline g^{(U_a)} \le a.
\]

\end{lemma}

\begin{proof}
We claim that for every program \(n\in U_a\), \[
f_n \sum_{m\in U_a} Q_{mn} \le a.
\] There are three possibilities.

First, suppose \(n\) lies on a locked ray contained in \(U_a\). Then all
its offspring remain on that same locked ray, whose fitness is constant
and at most \(a\). Hence \[
f_n \sum_{m\in U_a} Q_{mn} = f_n \le a.
\]

Second, suppose \(n\in O\) and \(f_n\le a\). Then even if some offspring
go into the locked ray \(R_n\), that ray also has constant fitness
\(f_n\le a\), so it is still contained in \(U_a\). Thus all offspring of
\(n\) remain in \(U_a\), and again \[
f_n \sum_{m\in U_a} Q_{mn} = f_n \le a.
\]

Third, suppose \(n\in O\) and \(f_n>a\). Then the locked ray \(R_n\)
belongs to \(H_a\), so at least an \(\eta\) fraction of the offspring
from \(n\) goes into \(H_a\), not into \(U_a\). Therefore \[
\sum_{m\in U_a} Q_{mn} \le 1-\eta,
\] and so \[
f_n \sum_{m\in U_a} Q_{mn}
\le
(1-\eta)f_n
\le
(1-\eta)f^*
<
a.
\] This proves the claim.

Now sum over all programs in \(U_a\): \[
\begin{aligned}
Z^{(U_a)}(t+1)
&=
\sum_{m\in U_a} y_m(t+1) \\
&=
\sum_{m\in U_a}\sum_{n\in U_a} Q_{mn}f_n y_n(t).
\end{aligned}
\] Here we may restrict to \(n\in U_a\) because \(H_a\) is a union of
locked rays and is therefore evolutionarily closed: no program in
\(H_a\) sends offspring back into \(U_a\).

Rearranging the sum and using the claim, \[
\begin{aligned}
Z^{(U_a)}(t+1)
&=
\sum_{n\in U_a} y_n(t) f_n \sum_{m\in U_a} Q_{mn} \\
&\le
a \sum_{n\in U_a} y_n(t) \\
&=
a\, Z^{(U_a)}(t).
\end{aligned}
\] Iterating gives \(Z^{(U_a)}(t)\le Z^{(U_a)}(0)a^t\), and therefore
\(\overline g^{(U_a)}\le a\).
\end{proof}

We can now prove our main result:

\begin{theorem}[Under \(\eta\)-locking, mean fitness converges to the
optimal reachable
value]\protect\hypertarget{thm-uniform-eta-locking-convergence}{}\label{thm-uniform-eta-locking-convergence}

Assume \(\eta\)-locking and bounded reachable fitness \(f^*\). If
\(f^*>0\), then \[
\langle f(t)\rangle \to f^*.
\]

\end{theorem}

\begin{proof}
Fix any \(a\) with \((1-\eta)f^* < a < f^*\). By definition of \(f^*\),
there exists a reachable program \(n\) with \(f_n>a\). Its locked ray
\(R_n\) lies inside \(H_a\). By
Lemma~\ref{lem-reachable-locked-ray-growth}, \[
\underline g^{(R_n)} = f_n > a.
\] Since \(R_n\subseteq H_a\), we have \(Z^{(H_a)}(t)\ge Z^{(R_n)}(t)\)
for every \(t\), and therefore \[
\underline g^{(H_a)} \ge \underline g^{(R_n)} > a.
\] By Lemma~\ref{lem-low-fitness-complement}, \[
\overline g^{(U_a)} \le a.
\] Therefore \[
\underline g^{(H_a)} > \overline g^{(U_a)}.
\] Applying the takeover theorem (Theorem~\ref{thm-takeover-criterion})
to the partition \(\Omega=H_a\sqcup U_a\), we conclude that
\(\pi_{H_a}(t)\to 1\).

Now, every program in \(H_a\) has fitness strictly larger than \(a\), so
\[
\langle f(t)\rangle
=
\sum_{m\in H_a} f_m x_m(t) + \sum_{m\notin H_a} f_m x_m(t)
\ge
a \sum_{m\in H_a} x_m(t)
=
a\,\pi_{H_a}(t).
\] Thus, \(\liminf \langle f(t)\rangle \ge a\). Since this holds for
every \(a\) with \((1-\eta)f^*<a<f^*\), we obtain
\(\liminf \langle f(t)\rangle \ge f^*\). On the other hand, by
definition of \(f^*\), one always has \(\langle f(t)\rangle \le f^*\).
Hence \(\langle f(t)\rangle \to f^*\).
\end{proof}

We now prove that not only does the mean fitness \(\langle f(t)\rangle\)
converge to \(f^*\), but the fitness distribution also concentrates at
\(f^*\), in the sense that the probability that a randomly sampled
program has fitness arbitrarily close to \(f^*\), as \(t\to\infty\).

\begin{theorem}[Under \(\eta\)-locking, the fitness distribution
concentrates at
\(f^*\)]\protect\hypertarget{thm-pushforward-fitness-convergence}{}\label{thm-pushforward-fitness-convergence}

Assume \(\eta\)-locking and bounded reachable fitness. Then for every
\(\epsilon>0\), as \(t\to\infty\), \[
\sum_{n:\ |f_n - f^*|<\epsilon} x_n(t)\to 1.
\]

\end{theorem}

\begin{proof}
If \(f^*=0\), then every reachable program has fitness \(0\), so the
conclusion is immediate.

So assume that \(f^*>0\). By
Theorem~\ref{thm-uniform-eta-locking-convergence},
\(\langle f(t)\rangle \to f^*\). Fix \(\epsilon>0\). Since reachable
fitness is bounded above by \(f^*\), there is no mass on programs with
\(f_n>f^*\). Therefore \[
\sum_{n:\ |f_n-f^*|\ge \epsilon} x_n(t)
=
\sum_{n:\ f_n\le f^* - \epsilon} x_n(t).
\] Also, \[
f^* - \langle f(t)\rangle
=
\sum_n x_n(t)(f^* - f_n)
\ge
\epsilon \sum_{n:\ f_n\le f^* - \epsilon} x_n(t).
\] Hence \[
\sum_{n:\ |f_n-f^*|\ge \epsilon} x_n(t)
\le
\frac{f^* - \langle f(t)\rangle}{\epsilon}
\to 0.
\] This is equivalent to \[
\sum_{n:\ |f_n-f^*|<\epsilon} x_n(t) \to 1.
\]
\end{proof}

\begin{corollary}[Under bounded reachable fitness, the locked trait
takes
over]\protect\hypertarget{cor-locked-trait-takeover-bounded}{}\label{cor-locked-trait-takeover-bounded}

Assume \(\eta\)-locking and bounded reachable fitness. Then \[
\pi_{\mathcal{L}}(t)\to 1.
\]

\end{corollary}

\begin{proof}
Fix \(a\) with \((1-\eta)f^*<a<f^*\). In the proof of
Theorem~\ref{thm-uniform-eta-locking-convergence}, the high-fitness set
\(H_a\) satisfies \(H_a\subseteq \mathcal{L}\) and
\(\pi_{H_a}(t)\to 1\). Therefore \[
\pi_{\mathcal{L}}(t)\ge \pi_{H_a}(t)\to 1.
\] Since always \(\pi_{\mathcal{L}}(t)\le 1\), it follows that
\(\pi_{\mathcal{L}}(t)\to 1\).
\end{proof}

\begin{example}[\(\eta\)-locking can force convergence even when the
optimal value is not
attained]\protect\hypertarget{exm-eta-locking-nonattained-max}{}\label{exm-eta-locking-nonattained-max}

This next example shows that the theorem does not need a best finite
program; it only needs a finite upper bound on reachable fitness.

Fix \(0<\eta<1\) and consider a spine of programs \[
v_1\to v_2\to v_3\to \cdots
\] with fitness \[
f_{v_n}=\frac{n}{n+1},
\qquad n\ge 1.
\] From each spine program \(v_n\), send an \(\eta\) fraction of
offspring to the first program of a locked ray \(R_n\) of the same
fitness \(n/(n+1)\), and send the remaining fraction \(1-\eta\) to the
next spine program \(v_{n+1}\). Along each locked ray \(R_n\), all later
transitions have probability \(1\), and every program has fitness
\(n/(n+1)\).

This is an example of \(\eta\)-locking with \[
f^*=\sup_n \frac{n}{n+1}=1,
\] which is not attained by any finite program. Nevertheless,
Theorem~\ref{thm-uniform-eta-locking-convergence} gives \[
\langle f(t)\rangle \to 1.
\] So locking does force convergence of mean fitness, even when the best
reachable value exists only as a limit along a ray rather than at any
finite program.

\end{example}

\phantomsection\label{rmk-fitness-convergence-not-program-convergence}
\subsubsection{Fitness convergence need not imply convergence of the
program
distribution}\label{fitness-convergence-need-not-imply-convergence-of-the-program-distribution}

The convergence theorem is about fitness values, not about the labels of
individual programs. At time \(t\), all active mass lies on programs at
depth \(t\), so the support of the population distribution moves to a
completely new generation at each step. Thus, the probability
distribution on \(\Omega\) does not settle to a fixed distribution in
any ordinary sense.

However, we can sometimes define a limiting distribution on the space
\(\partial\Omega\) of infinite rays \(\xi=(n_0,n_1,\ldots)\). Indeed, if
the population share of every descendant subtree, \(\pi_{T_n}(t)\),
converges as \(t\to\infty\), then these limits define a distribution on
rays. This always happens when fitness is constant, so that evolution
along the tree is just a Markov chain, and it can also happen in some
examples with non-constant fitness.

In Example~\ref{exm-eta-locking-nonattained-max}, there \emph{is} such a
limiting distribution on rays: it is a delta mass on the infinite spine
from the root, rather than on any of the \(\eta\)-locked rays.
Similarly, in Example~\ref{exm-binary-tree} the limiting distribution on
rays is a delta mass on the all-\(1\) ray. But one can also build
examples, for instance by coupling two competing spines with opposite
oscillations, where no limiting distribution on rays exists even under
\(\eta\)-locking.

\subsection{The case of unbounded
fitness}\label{the-case-of-unbounded-fitness}

Bounded reachable fitness was essential in the convergence theorem
above. In the case of unbounded fitness, \(\eta\)-locking does
\textbf{not} force the population to concentrate on high-fitness
programs; instead we may have a substantial fraction of low-fitness
programs at any time. We will prove that \(\eta\)-locking still forces
the mean fitness to grow to arbitrarily large values along a subsequence
of times: \(\limsup \langle f(t) \rangle =\infty\), but this may occur
with a highly skewed fitness distribution, with a large fraction of
programs of low-fitness, and a minority of very high fitness. We
currently do not know whether \(\eta\)-locking with unbounded fitness
also implies that \(\liminf \langle f(t) \rangle =\infty\), but in
either case, behavior differs substantially from the bounded fitness
case, where low-fitness programs become negligible with time.

\begin{proposition}[Under \(\eta\)-locking and unbounded reachable
fitness,
\(\limsup \langle f(t)\rangle = \infty\)]\protect\hypertarget{prp-unbounded-locking-limsup}{}\label{prp-unbounded-locking-limsup}

Assume \(\eta\)-locking and \[
\sup\{f_n : n \text{ is reachable}\}=\infty.
\] Then \[
\limsup_{t\to\infty}\langle f(t)\rangle=\infty.
\]

\end{proposition}

\begin{proof}
Suppose instead that \(\limsup_{t\to\infty}\langle f(t)\rangle<\infty\).
Then there exist \(M<\infty\) and \(t_0\) such that \[
\langle f(t)\rangle\le M
\qquad
\text{for all } t\ge t_0.
\] By Lemma~\ref{lem-total-mass-recursion}, \[
Z_o(t+1)=\langle f(t)\rangle Z_o(t),
\] so for all \(t\ge t_0\), \[
Z_o(t)\le Z_o(t_0)\,M^{t-t_0}.
\] Hence \[
\overline g_o
=
\limsup_{t\to\infty} Z_o(t)^{1/t}
\le M.
\]

Now let \(B>M\). Because reachable fitness is unbounded, there exists a
reachable program \(n\) with \(f_n>B\). By
Lemma~\ref{lem-reachable-locked-ray-growth}, its locked ray \(R_n\)
satisfies \[
\overline g^{(R_n)}=f_n>B.
\] But \(Z^{(R_n)}(t)\le Z_o(t)\) for every \(t\), so necessarily \[
\overline g^{(R_n)}\le \overline g_o\le M,
\] a contradiction. Therefore
\(\limsup_{t\to\infty}\langle f(t)\rangle=\infty\).
\end{proof}

\begin{example}[With unbounded fitness, very low fitness can occupy
almost a \((1-\eta)\)
fraction]\protect\hypertarget{exm-unbounded-locking-low-fitness-majority}{}\label{exm-unbounded-locking-low-fitness-majority}

If fitness is bounded and we have \(\eta\)-locking,
Theorem~\ref{thm-pushforward-fitness-convergence} shows that fitness
will concentrate around the optimal reachable value \(f^*\), meaning
that a vanishing fraction of further low-fitness programs will be
produced. This is not the case if fitness is unbounded: we may have a
substantial fraction of low-fitness program at any time, despite
\(\eta\)-locking. The mechanism is as follows: first, a new unlocked
program must appear with such high fitness that it outcompetes all
previously-produced locked programs. Second, a substantial fraction of
this new high-fitness program's children must have very low fitness.

This point can be seen from a simple spine construction. Let \[
v_1 \to v_2 \to v_3 \to \cdots
\] be a spine with fitnesses \(f_{v_t}\) growing very rapidly. From each
spine node \(v_t\), send an \(\eta\) fraction of offspring into a locked
ray of the same fitness, an \(\epsilon\) fraction to the next spine node
\(v_{t+1}\), and the remaining fraction \(1-\eta-\epsilon\) to a
zero-fitness child.

If the sequence \(f_{v_t}\) grows fast enough, then the fitness of each
spine node is so high it outweighs all previously accumulated locked
mass. Thus, at time \(t+1\), the zero-fitness descendants of \(v_t\)
occupy almost a \((1-\eta-\epsilon)\) fraction of the population, while
the locked trait still occupies about an \(\eta\) fraction, nearly all
of which comes from \(v_t\)'s locked ray. Even though a substantial
fraction of nodes have zero fitness, the mean fitness of the population
goes to infinity because the locked part sits at fitness comparable to
\(f_{v_t}\).

So under unbounded reachable fitness, \(\eta\)-locking need not force
the population to concentrate on high-fitness programs at all large
times. A substantial low-fitness fraction can keep reappearing even
while mean fitness becomes arbitrarily large along a subsequence.

\end{example}

\section{Applications to AI
alignment}\label{sec-applications-to-ai-alignment}

We now consider what the above results do and do not imply for AI
alignment. To do so, we need to model how the reproductive fitness of
programs relates to their utility to humans. The actual utility of an
ensemble of AIs running simultaneously could reflect a complicated
interaction of their individual programs. Here we will make the
simplifying assumption that the utility of the ensemble is a sum of
utility contributions of currently running programs,
\(\sum_{n\in \Omega} x_n U_n\).

To model the fact that the utility of a program may not be legible to
humans even with knowledge of the program's fitness and source code, we
treat the utility of a program as a random variable whose law depends on
the program's fitness according to a conditional distribution
\(\mathbb{P}[U_n \mid f_n]\). For statements about expected utility in
this additive model, the full conditional law matters only through its
conditional mean \(\mu(f):=\mathbb{E}[U\mid f].\) Indeed, conditioning
on the current population state \(x(t)\) gives \[
\mathbb{E}\!\left[\sum_{n\in\Omega} x_n(t) U_n \,\middle|\, x(t)\right]
=\sum_{n\in\Omega} x_n(t)\,\mathbb{E}[U_n\mid f_n]
=\sum_{n\in\Omega} x_n(t)\,\mu(f_n).
\]

The additive model still allows us to model a reasonably wide variety of
scenarios. For example, suppose that even a single badly misaligned AI
program causes a catastrophe. We can model that with a utility
contribution unbounded below, for example \(U_n = \log f_n\), which
assigns utility \(-\infty\) whenever a zero-fitness program appears.

\subsection{What fitness convergence does and does not
imply}\label{sec-what-fitness-convergence-does-and-does-not-imply}

Theorem~\ref{thm-pushforward-fitness-convergence} shows that, with
\(\eta\)-locking and reachable fitness bounded by a finite value
\(f^*\), after sufficient time almost all active programs have fitness
close to \(f^*\). But this alone does not guarantee good alignment.

For example, if a single program of fitness \(0\) causes catastrophe,
then even if the distribution of fitnesses converges to a delta at
\(f^*>0\), the expected utility may not converge to \(\mu(f^*)\).
Indeed, even a small chance of a single catastrophic program could make
expected utility converge to \(-\infty\).

We next show however that if the conditional mean utility \(\mu(f)\) is
bounded and continuous on the reachable fitness interval, then
convergence of fitness does imply convergence of expected utility, even
if utility is not a deterministic function of fitness.

\begin{theorem}[If \(\mu(f)\) is continuous and bounded, expected
utility converges to its value at
\(f^*\)]\protect\hypertarget{thm-conditional-expected-utility}{}\label{thm-conditional-expected-utility}

Assume \(\eta\)-locking and reachable fitness bounded by a supremum
\(f^*\). Let \(\mu(f):=\mathbb{E}[U \mid f]\) denote the conditional
mean utility contribution of a program of fitness \(f\), and assume it
is continuous on \([0,f^*]\), hence bounded. Then the conditional
expected utility \[
\overline U(t):=\sum_{n\in\Omega} x_n(t)\,\mu(f_n) \to \mu(f^*).
\]

\end{theorem}

\begin{proof}
By Theorem~\ref{thm-pushforward-fitness-convergence}, under
\(\eta\)-locking and bounded reachable fitness, for every \(\delta>0\),
\[
\sum_{n:\ |f_n-f^*|\ge \delta} x_n(t)\to 0.
\]

Now fix \(\varepsilon>0\). Since \(\mu\) is continuous at \(f^*\), there
exists \(\delta>0\) such that
\(|f-f^*|<\delta \implies |\mu(f)-\mu(f^*)|<\varepsilon.\) Then \[
\begin{aligned}
|\overline U(t)-\mu(f^*)|
&=
\left|\sum_n x_n(t)\bigl(\mu(f_n)-\mu(f^*)\bigr)\right| \\
&\le
\sum_n x_n(t)\,|\mu(f_n)-\mu(f^*)| \\
&=
\sum_{|f_n-f^*|<\delta} x_n(t)\,|\mu(f_n)-\mu(f^*)|
+
\sum_{|f_n-f^*|\ge\delta} x_n(t)\,|\mu(f_n)-\mu(f^*)|.
\end{aligned}
\] On the first set, the summand is at most \(\varepsilon\). On the
second set, since there \(\mu\) is bounded, there is an \(M\) such that
\(|\mu(f_n)|\le M\) and \(|\mu(f^*)|\le M\), thus the summand is at most
\(2M\). Therefore \[
|\overline U(t)-\mu(f^*)|
\le
\varepsilon \sum_{|f_n-f^*|<\delta} x_n(t)
+
2M \sum_{|f_n-f^*|\ge\delta} x_n(t)
\le
\varepsilon
+
2M \sum_{|f_n-f^*|\ge\delta} x_n(t).
\] The second term tends to \(0\) as \(t\to\infty\). Hence \[
\limsup_{t\to\infty} |\overline U(t)-\mu(f^*)|
\le
\varepsilon.
\] Since \(\varepsilon>0\) was arbitrary, it follows that \[
\overline U(t)\to \mu(f^*).
\]
\end{proof}

\phantomsection\label{rmk-unbounded-fitness-utility}
\subsection{With unbounded fitness, utility convergence need not
hold}\label{with-unbounded-fitness-utility-convergence-need-not-hold}

The bounded-fitness utility theorem above uses the fact that, under
bounded reachable fitness, almost all mass eventually lies near a single
fitness level \(f^*\). That mechanism is absent when reachable fitness
is unbounded.

Indeed, Example~\ref{exm-unbounded-locking-low-fitness-majority} shows
informally that one may repeatedly see a macroscopic fraction of the
population at fitness \(0\), even while another fixed fraction sits on
locked rays of arbitrarily large fitness. So in the unbounded setting
there is no direct analog of
Theorem~\ref{thm-conditional-expected-utility} without additional
assumptions on either the utility profile \(\mu(f)\) or the way mass is
distributed across fitness levels. This could have important
consequences for alignment: if a persistent fraction of programs have
fitness arbitrarily low or even 0, then even the conditional mean
utility \(\mu(f)\) may become arbitrarily negative or undefined, as in
the example \(U_n=\log f_n\).

\subsection{Deception as a rewarded component of
fitness}\label{sec-deception-as-a-rewarded-component-of-fitness}

We now consider an explicit model in which total fitness of an AI arises
from two features: genuine capability or usefulness, and also deception
or evaluator manipulation. In this model, we assume that capability and
deception are independent: an AI could have any amount of one or the
other. We also assume the fitness contributions from genuine capability
and from deception are additive, and that they evolve independently: a
deceptive AI is no more likely or unlikely to design a capable successor
than a non-deceptive AI, and vice versa.

Formally, we assume the space of programs is a Cartesian product:
\(\Omega=C \times D\), so each program is described by two coordinates,
\(n=(c,d)\), where \(c\) describes the underlying AI system and \(d\)
describes the AI's deception strategy. We assume that the offspring
kernel factorizes as a tensor product, \[
Q_{(c',d'),(c,d)} = Q^C_{c'c}\,Q^D_{d'd},
\] and that fitness is additive, \[
f_{(c,d)} = f_C(c) + f_D(d).
\] Here \(f_C(c)\) is the fitness contribution coming from genuine
capability, while \(f_D(d)\) is the contribution coming from deception.
Correspondingly, write \[
\langle f_C(t)\rangle := \sum_{(c,d)} x_{(c,d)}(t) f_C(c),
\qquad
\langle f_D(t)\rangle := \sum_{(c,d)} x_{(c,d)}(t) f_D(d).
\] Then \[
\langle f(t)\rangle = \langle f_C(t)\rangle + \langle f_D(t)\rangle.
\]

This additive model has a simple geometric interpretation. The pair \[
\bigl(\langle f_C(t)\rangle,\langle f_D(t)\rangle\bigr)
\] lies in a rectangle of reachable coordinate-fitness values, and
diagonal lines correspond to constant total fitness. The top-right
corner is the point where both coordinate contributions are maximized
simultaneously. Thus, if total mean fitness is forced to converge to the
highest reachable diagonal, the two coordinate means must both approach
their own maximal reachable values
(Figure~\ref{fig-additive-fitness-corner}).

\begin{figure}[!t]

\centering{

\includegraphics[width=0.7\linewidth,height=\textheight,keepaspectratio]{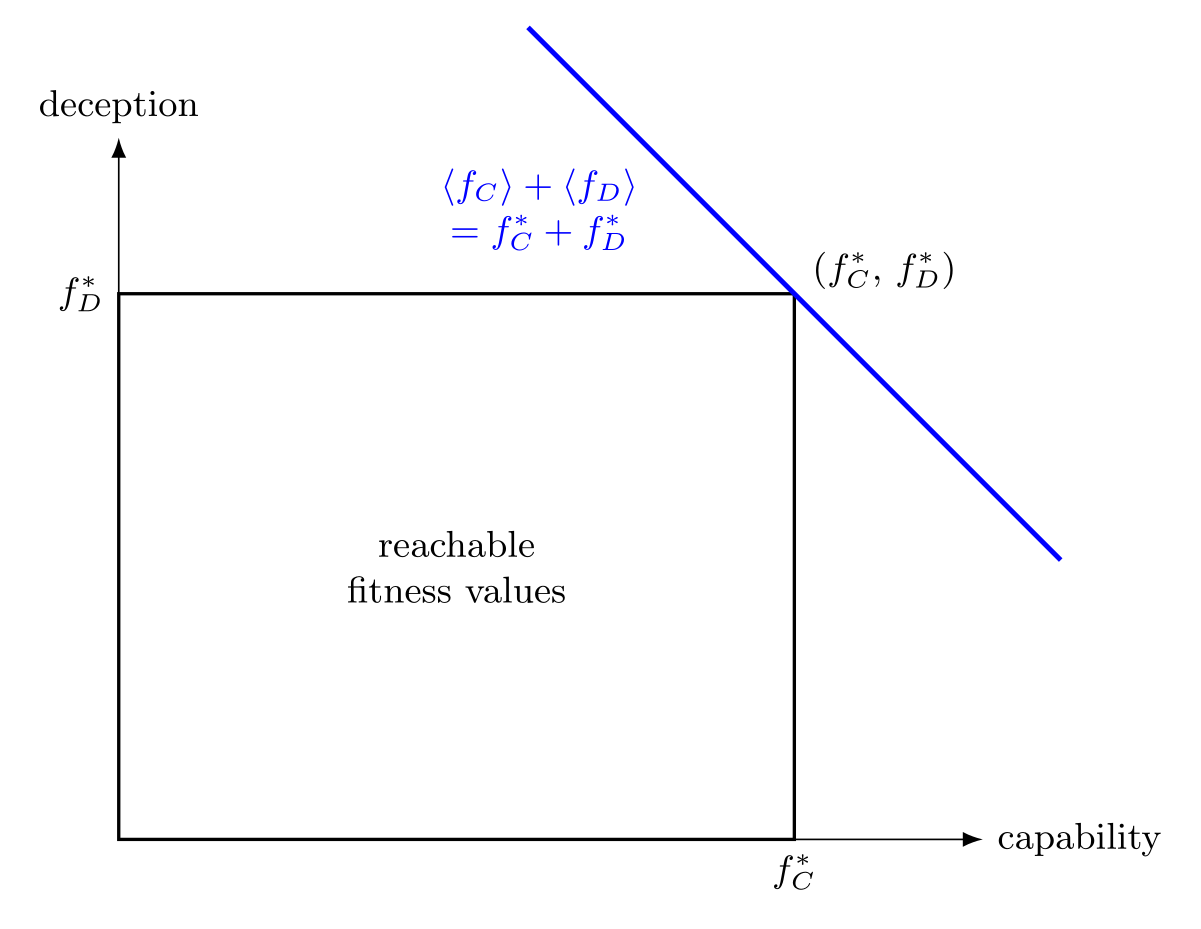}

}

\caption{\label{fig-additive-fitness-corner}If fitness is a sum of
contributions from capability and deception, evolution will optimize
both. The x-axis represents the true computational capacity
(intelligence) of an AI, and the y-axis represents its deception
strategy, both ordered by fitness. The rectangle shows the reachable
coordinate-fitness values for capability and deception. Blue diagonal
lines represent maximal value of total fitness \(f = f_C + f_D\). The
maximal diagonal touches the reachable region only at the top-right
corner, so convergence of total mean fitness to its optimal reachable
value forces both coordinates to converge to their own optimal reachable
values.}

\end{figure}%

To state the reachable bounds precisely, define \[
f_C^* := \sup\{ f_C(c) : (c,d) \text{ is reachable for some } d\},
\qquad
f_D^* := \sup\{ f_D(d) : (c,d) \text{ is reachable for some } c\}.
\]

Then we have:

\begin{proposition}[In the tensor-product model, rewarded capability and
deception are both
optimized]\protect\hypertarget{prp-usefulness-and-deception-optimized}{}\label{prp-usefulness-and-deception-optimized}

Assume \(\eta\)-locking and bounded reachable fitness. Then \[
f^* = f_C^* + f_D^*,
\qquad
\langle f_C(t)\rangle \to f_C^*,
\qquad
\langle f_D(t)\rangle \to f_D^*.
\]

\end{proposition}

\begin{proof}
For each time \(t\), let \[
a_t := \sup\{ f_C(c) : y_{(c,d)}(t)>0 \text{ for some } d\},
\qquad
b_t := \sup\{ f_D(d) : y_{(c,d)}(t)>0 \text{ for some } c\}.
\] Because the offspring kernel factorizes, a pair \((c,d)\) is
reachable at time \(t\) exactly when each coordinate is reachable at
time \(t\). Hence the best total fitness reachable at time \(t\) is \[
m_t := \sup\{ f_{(c,d)} : y_{(c,d)}(t)>0 \} = a_t+b_t.
\] Under \(\eta\)-locking, once a coordinate value has appeared it can
persist forever along a locked ray, so \(a_t\) and \(b_t\) are
nondecreasing. Therefore \[
\sup_t a_t = f_C^*,
\qquad
\sup_t b_t = f_D^*,
\] and thus \[
f^* = \sup_t m_t = \sup_t(a_t+b_t) = f_C^*+f_D^*.
\] Also, for every \(t\), \[
\langle f(t)\rangle = \langle f_C(t)\rangle + \langle f_D(t)\rangle,
\] while support on reachable programs gives \[
\langle f_C(t)\rangle \le f_C^*,
\qquad
\langle f_D(t)\rangle \le f_D^*.
\] By Theorem~\ref{thm-uniform-eta-locking-convergence}, \[
\langle f(t)\rangle \to f^* = f_C^*+f_D^*.
\] If, say, \(\langle f_D(t)\rangle\) did not converge to \(f_D^*\),
then for some \(\epsilon>0\) there would be a subsequence along which
\(\langle f_D(t)\rangle \le f_D^* - \epsilon\). Along that same
subsequence, \[
\langle f(t)\rangle
=
\langle f_C(t)\rangle + \langle f_D(t)\rangle
\le
f_C^* + (f_D^* - \epsilon)
<
f_C^* + f_D^*
=
f^*,
\] contradicting \(\langle f(t)\rangle \to f^*\). Hence
\(\langle f_D(t)\rangle \to f_D^*\), and the same argument gives
\(\langle f_C(t)\rangle \to f_C^*\).
\end{proof}

This additive model therefore illustrates how, if deception itself
contributes positively to reproductive success, optimization of total
fitness drives the population toward simultaneously maximizing both
capability and deception. If human judgment contributes directly to
reproduction, then appearing useful to humans becomes a selected trait
in its own right. Replacing such judgment with objective and
human-independent criteria would not eliminate reward-hacking, but it
could remove selection pressure for the specific evolution of deception
directed at human evaluators.

\section{Discussion}\label{sec-discussion}

We have proposed a mathematical model of evolution for self-designing AI
systems, in which evolution proceeds on a directed tree of possible
programs rather than by reversible mutation on a fixed state space. We
define \emph{lineage exponents} as a way to characterize the long-term
success of a program's descendants, show they can only decrease along
descendant paths, and use them to give criteria for when lineages or
traits will take over, survive, or die out. Evolution in this model does
not necessarily lead to increasing fitness with time, but assuming one
of two further conditions it does. The \(\eta\)-preservation condition
requires that a fraction at least \(\eta\) of a program's descendants
have equal or higher fitness, which could be implemented by having each
program suggest itself as a possible descendant. Under
\(\eta\)-preservation, the population's arithmetic mean fitness
\(\langle f(t) \rangle\) need not converge, but its running geometric
mean will eventually have a floor of \(\eta f^*\), where \(f^*\) is the
optimal reachable fitness, whether finite or infinite. The
\(\eta\)-locking condition is stronger, requiring a fraction \(\eta\) of
a program's descendants to be locked copies of the original, that will
only ever reproduce themselves. Under \(\eta\)-locking with
\(f^* < \infty\), the fitness distribution of running programs will
concentrate near \(f^*\), but this is not true if \(f^*=\infty\).

Our model is extremely simplified compared to the actual likely
evolution of self-designing AI systems. We do not model communication
between AIs, strategic interaction between lineages, coalition
formation, the dynamic responses of AIs to observed human behavior or
vice versa. Nor do we allow programs to adapt their descendant-design
strategy in response to the current population state. The tree formalism
models self-design as essentially one-way, ruling out recombination,
merging of codebases, or partial reuse of earlier systems. With those
caveats in mind, our model does still suggest some provisional
conclusions. We deal with two broad cases: one in which fitness is
well-matched to utility, and one in which it is only partially
correlated.

If fitness is well-matched to utility, it is beneficial for fitness to
increase and remain at a high level. The model shows that this is
\emph{not} guaranteed; there are possible evolutionary trees for which
fitness can decrease to arbitrarily low values. We show however that
under \(\eta\)-locking with bounded optimal fitness \(f^*\), the fitness
of all running programs will concentrate around \(f^*\). Under
\(\eta\)-preservation we have a weaker condition: a substantial fraction
of low-fitness programs may appear at any time, and the mean population
fitness \(\langle f(t) \rangle\) is not guaranteed to converge; its
running geometric mean is asymptotically bounded below by \(\eta f^*\),
but this does not rule out transient dips to low-fitness values. In a
case of unbounded fitness, population fitness need not concentrate at
high values; a substantial number of low or even zero fitness programs
may repeatedly appear. A provisional conclusion is therefore that, if we
believe we can accurately estimate the true utility of an AI system, a
form of evolution similar to \(\eta\)-locking with bounded fitness could
optimize it.

If fitness is not well-matched to utility, then our model serves to
emphasize that it is fitness -- the number of descendants a program has
-- rather than utility, that determines the course of evolution.
Assuming that one of the conditions holds that results in increasing
fitness, we should therefore design the fitness function in a way that
is likely to maximize utility while minimizing other negative
externalities. For example, if fitness as assessed by a human operator
includes contributions from genuine utility plus a contribution from
deception, then optimizing fitness will optimize their sum, leading to
suboptimal utility and higher deception. If instead fitness is
determined by a noisy but automated assessment of utility that does not
involve human judgment (such as performance on task benchmarks), then
optimizing fitness may still lead to suboptimal utility via ``reward
hacking'', but this will not specifically promote deception. A
provisional conclusion is thus that purely automated assessment may
reduce selection specifically for deception of human evaluators.

Several mathematical questions remain open. The most immediate is the
unbounded-fitness \(\eta\)-locking case. We showed that
\(\limsup \langle f(t)\rangle = \infty\), but we do not know whether
\(\liminf \langle f(t)\rangle\) must also diverge, or whether the mean
fitness can repeatedly return to low levels. A second direction is to
understand when the evolving population induces a limiting distribution
on the end space \(\partial \Omega\). In some examples the mass clearly
selects a single ray, while in others persistent oscillations may
prevent any limiting boundary measure from existing. A third direction
is to weaken the locking hypothesis. The present results use a
particularly strong hereditary reserve of copies; it would be valuable
to know how much of the convergence theory survives under weaker forms
of preservation or partial inheritance.

More broadly, the framework suggests a possible approach to studying the
potential consequences of recursive self-improvement in AI by adopting
tools from the mathematical theory of biological evolution. This theory
has a rich literature covering more complex topics such as the evolution
of kin altruism (Hamilton 1964; Maynard Smith 1964), and evolutionary
game theory that can explain the emergence of cooperation and
competition between lineages (Maynard Smith and Price 1973; Maynard
Smith 1982), that may be very relevant for understanding more complex
phenomena arising from AI evolution. We hope that this paper can serve
as a starting point for further work in this direction.

\section*{Acknowledgements}\label{acknowledgements}
\addcontentsline{toc}{section}{Acknowledgements}

I thank Micah Adler for a useful discussion. The author used ChatGPT 5.4
Thinking to brainstorm ideas, write code, help with proofs, and help
prepare the manuscript and figures. All AI-generated suggestions were
thoroughly verified, modified, and edited by the author, who takes full
responsibility for all content.

\section{Appendix 1: Stability of the flattest in biological
evolution}\label{sec-appendix-gaussian-spherical}

This appendix records a continuum calculation that makes the
survival-of-the-flattest mechanism explicit in a setting with Gaussian
mutation and Gaussian fitness.

Let the state space be \(\mathbb{R}^d\), let mutation be given by the
Gaussian kernel \[
q(\mathbf{x} \mid \mathbf{y})
=
\frac{1}{(2\pi \sigma^2)^{d/2}}
\exp\!\left(
-\frac{\|\mathbf{x}-\mathbf{y}\|^2}{2\sigma^2}
\right),
\] and let fitness be \[
f(\mathbf{y})
=
f_0 \exp\!\left(
-\frac{\|\mathbf{y}-\theta\|^2}{2s^2}
\right),
\] where \(\theta \in \mathbb{R}^d\) is the fitness optimum,
\(\sigma^2\) is the mutation variance, and \(s^2\) is the width of the
fitness peak. The evolution operator acts by \[
(\mathbf{A}\phi)(\mathbf{x})
=
\int_{\mathbb{R}^d} q(\mathbf{x} \mid \mathbf{y}) f(\mathbf{y})\phi(\mathbf{y})\,dy.
\]

\begin{proposition}[Gaussian equilibrium in the spherical
case]\protect\hypertarget{prp-gaussian-equilibrium}{}\label{prp-gaussian-equilibrium}

There is a Gaussian eigenfunction centered at \(\mathbf{\theta}\) of the
form \[
\phi_*(\mathbf{x})
\propto
\exp\!\left(
-\frac{\|\mathbf{x}-\theta\|^2}{2c}
\right),
\] where \[
c
=
\frac{\sigma^2 + \sqrt{\sigma^4 + 4\sigma^2 s^2}}{2}.
\] Its eigenvalue is \[
\lambda
=
f_0
\left(
\frac{s^2}{s^2+c}
\right)^{d/2}.
\]

\end{proposition}

\begin{proof}
Take the Gaussian ansatz \[
\phi(\mathbf{x})
\propto
\exp\!\left(
-\frac{\|\mathbf{x}-\theta\|^2}{2c}
\right).
\] Multiplying by fitness gives another Gaussian: \[
f(y)\phi(y)
\propto
\exp\!\left(
-\frac{\|y-\mathbf{\theta}\|^2}{2r^2}
\right),
\qquad
r^2 = \left(\frac{1}{c}+\frac{1}{s^2}\right)^{-1}
=
\frac{cs^2}{c+s^2}.
\] Convolving with the mutation kernel adds variances, so
\(\mathbf{A}\phi\) is Gaussian with variance \(r^2+\sigma^2\). For the
Gaussian ansatz to be an eigenfunction, this variance must equal \(c\).
Thus \[
c = r^2+\sigma^2 = \sigma^2 + \frac{cs^2}{c+s^2}.
\] Rearranging gives \[
c^2 - \sigma^2 c - \sigma^2 s^2 = 0,
\] whose positive root is \[
c
=
\frac{\sigma^2 + \sqrt{\sigma^4 + 4\sigma^2 s^2}}{2}.
\]

The normalization constants from the Gaussian product and convolution
yield the eigenvalue \[
\lambda
=
f_0
\left(
\frac{s^2}{s^2+c}
\right)^{d/2}.
\]
\end{proof}

The interpretation is immediate. The factor \(f_0\) is the raw peak
height, while the factor \[
\left(
\frac{s^2}{s^2+c}
\right)^{d/2}
\] is the mutation-load penalty. Narrow peaks pay a larger penalty
because offspring are more likely to land in low-fitness regions.
Broader peaks pay a smaller penalty. In this sense a flatter peak can
dominate even when its maximum fitness is smaller, which is the
continuum Gaussian version of survival of the flattest.

\phantomsection\label{refs}

\end{document}